\documentclass[11pt]{article}

\usepackage{chngcntr} 
\usepackage{amsmath}
\usepackage{enumerate}
\usepackage{fancyhdr}
\usepackage{color}
\usepackage{listings}
\usepackage{graphicx}
\usepackage{amssymb}
\usepackage{mathtools}

\usepackage{hyperref}
\usepackage{booktabs}

\usepackage{algorithm}
\usepackage{algorithmic}

\usepackage[round]{natbib}
\usepackage{subfig}

\newtheorem{theorem}{Theorem}

\newtheorem{lemma}{Lemma}
\newtheorem{definition}{Definition}

\newtheorem{remark}{Remark}
\newtheorem{proposition}{Proposition}
\newtheorem{proof}{Proof}
\newtheorem{claim}{Claim}

\definecolor{lightgray}{gray}{0.5}

\usepackage[papersize={8.5in,11in},top=1in,bottom=0.75in,left=0.75in,right=0.75in]{geometry}
\usepackage[colorinlistoftodos]{todonotes}
\usepackage{enumitem}
\usepackage{tablefootnote}

\title{\textbf{Decentralized Competing Bandits in Non-Stationary Matching Markets}}

\author{Avishek Ghosh$^{\dagger}$, Abishek Sankararaman$^{\ddagger}$, Kannan Ramchandran$^{\star}$, Tara javidi$^{\dagger,\diamond}$ and Arya Mazumdar$^{\dagger}$  \vspace{2mm} \\
Halıcıoğlu Data Science Institute (HDSI), UC San Diego$^\dagger$ \\
\vspace{1.5mm}
AWS AI, Palo Alto, USA$^\ddagger$  \\
\vspace{1.5mm}
Electrical Engg. and Computer Sciences, UC Berkeley$^{\star}$\\
\vspace{1.5mm}
Electrical and Computer Engineering, UC San Diego$^\diamond$\\
email: \{a2ghosh,arya,tjavidi\}$@$ucsd.edu, abisanka$@$amazon.com, kannanr$@$eecs.berkeley.edu
}

\date{}
\begin{document}
\maketitle

\begin{abstract}
Understanding complex dynamics of two-sided online matching markets, where the demand-side agents compete to match with the supply-side (arms), has recently received substantial interest. To that end, in this paper, we introduce the framework of decentralized two-sided matching market under non stationary (dynamic) environments. We adhere to the serial dictatorship setting, where the demand-side agents have unknown and different preferences over the supply-side (arms), but the arms have fixed and known preference over the agents. We propose and analyze a decentralized and asynchronous learning algorithm, namely Decentralized Non-stationary Competing Bandits (\texttt{DNCB}), where the agents play  (restrictive) successive elimination type learning algorithms to learn their preference over the arms. The complexity in understanding such a system stems from the fact that the competing bandits choose their actions in an asynchronous fashion, and the lower ranked agents only get to learn from a set of arms, not \emph{dominated} by the higher ranked agents, which leads to \emph{forced exploration}. With carefully defined complexity parameters, we characterize this \emph{forced exploration} and obtain sub-linear (logarithmic) regret of \texttt{DNCB}. Furthermore, we validate our theoretical findings via experiments.
\end{abstract}

\vspace{-5mm}
\section{Introduction}
\label{sec:intro}
\vspace{-2.5mm}
Repeated decision making by multiple agents in a competitive and uncertain environment is a key characteristic of modern day, two sided markets, e.g., TaskRabbit, UpWork, DoorDash, etc. Agents often act in a decentralized fashion on these platforms, and understanding the induced dynamics is an important step before designing policies around how to operate such platforms to maximize various system objectives such as revenue, efficiency and equity of allocations \citep{johari2021matching, liu2020competing}. A body of recent work is aimed at understanding the decentralized learning dynamics in such matching markets \cite{ucbd3,liu2020competing,liu2021bandit,dai2021learning,dai2021multi,basu21a}. This line of work studies the matching markets introduced first by the seminal work of \cite{gale1962college}, under the assumption where the participants are not aware of their preference and learn it over time by participating in the market. A key assumption made in these studies is that the true preferences of the market participants are static over time, and thus can be learnt with repeated interactions.

Markets, however are seldom stationary and continuously evolving. Indeed, an active area of research in management sciences and operations research revolve around understanding the equilibrium properties in such evolving markets \cite{damiano2005stability,akbarpour2020thickness, kurino2020credibility, johari2021matching}. However, a central premise in this line of work is that the participants have exact knowledge over their preferences, and only need to optimize over other agents' competitive behaviour and future changes that may occur. In this work, we take a step towards bridging the two aforementioned lines of work. To be precise, we study the learning dynamics in markets where both the participants do not know their exact preferences and the unknown preferences are themselves \emph{smoothly varying} over time.

Conceptually, the seemingly simple addition of varying preferences invalidates the core premise of learning algorithms in a stationary environment (such as those in \citep{liu2021bandit, ucbd3}) where learning is guaranteed to get better with time as more samples can potentially be collected. In a dynamic environment, agents need to additionally trade-off collecting more samples by competing with other agents to have a refined estimate, with the possibility that the quantity to be estimated being stale and thus not meaningful.

\textbf{Model Overview:} The model we study consists of $N$ agents and $k \geq N$ resources or arms, where the agents  repeatedly make decisions of which arm to match with over a time horizon of $T$. The agents are globally ranked from $1$ through $N$. The agents are initially assumed to not know their rank. In each round, every agent chooses one of the $k$ arms to match with. Every arm that has one or more agents requesting for a match, \emph{allocates} itself to the highest ranking agent requesting a match\footnote{If $i < j$, then agent with rank $i$ is said to be higher ranked than agent $j$}, while \emph{blocking} all other requesting agents. If at time $t$, agent $j$ is matched to arm $\ell$, then agent $i$ sees a random reward independent of everything else with mean $\mu_{j,\ell,t}$. The agents that are blocked are notified of being blocked and receive $0$ reward. Moreover the agents are decentralized, i.e., make decisions on which arm to match is a function of the history of the arms chosen, arms matched and rewards obtained at that agent. 

The key departure from prior works of \cite{liu2020competing, liu2021bandit, ucbd3} is that the unknown arm-means between any agent $j$ and arm $\ell$ is time-varying, i.e., the mean is dependent on time $t$. We call our model \emph{smoothly varying}, because we impose the constraint that for all agents $j$ and arms $\ell$, and time $t$, $|\mu_{j,\ell,t} - \mu_{j,\ell,t+1}| \leq \delta$, for some known parameter $\delta$. However, we make no assumptions on the synchronicity of the markets, i.e., the environments of different agents can change arbitrarily with the only constraint that any arm-agent pair means does not change by more than $\delta$ in one time-step. 

\textbf{Why is this model challenging ?} Even in the single agent case without competitions, algorithms such as UCB \cite{auer2002finite} perform poorly compared to algorithms such as {\ttfamily SnoozeIT} \cite{krishnamurthy2021slowly}  that adapts to the varying arm means (c.f. Figure \ref{fig:single_agent}) in smoothly varying environments. The reason is that stationary algorithms such as UCB weighs all the samples collected thus far equally in identifying which arm to pull, while adaptive algorithms such as {\ttfamily SnoozeIt} weighs recent samples more than older samples in order to estimate the arm-mean at the current time point. This is exacerbated in a multi-agent competitive setup where agents need to decide whether to pull an arm that yielded good results in the past, but is facing higher competition at the present. 

We circumvent this problem by introducing the idea of \emph{forced exploration}. Since the environments across agents are time-varying possibly asynchronously, a lower ranked agent may be forced to explore and obtain linear regret, if any of the higher ranked agents are exploring. To build intuition, consider a $2$ agent system in which the higher ranked agent  is called Agent 1, and the other agent is Agent 2. Suppose, Agent $1$'s environment (i.e., arm-means) are volatile where the gap between the best and second best arm is small, while Agent $2$ has a more benign environment, where all arm-means are well separated and not varying with time. In this case, Agent $1$ will be forced to explore arms a lot as its environment is fluctuating with no clear best arm emerging. Since any collision implies that Agent $2$ will not receive a reward, Agent $2$ is also forced to explore and play sub-optimal arms to evade collision, even when it knows its own best arms. This phenomenon indeed also occurs in the stationary setting, albeit in the stationary setting, every agent knows that after an initial exploration time, all agents will ``settle" down and find their best arm. This is the concept of freezing time in \cite{ucbd3, basu21a}. In the dynamic setting however, the forced explorations can keep occurring repeatedly over time, as the agents environment changes.

\vspace{-2.5mm}
\subsection{Our Contributions}
\vspace{-2.5mm}
\subsubsection{Algorithms}
\vspace{-2mm}
We introduce a learning algorithm, {\ttfamily DNCB}, in which agents proceed in phases with asynchronous start and end-points, wherein in each phase, agents explore among those arms that are not currently preferred by higher-ranked agents, and subsequently exploit a good arm, for a dynamic duration of time in which the estimated best arm can remain to be optimal. The main algorithmic innovation is to  identify that the static synchronous arm-deletion strategy of UCB-D3 \cite{ucbd3}, can be coupled with {\ttfamily SnoozeIt} to yield a dynamic, asynchronous explore-exploit type algorithm for non-stationary bandits. 
\vspace{-2.5mm}
\subsubsection{Technical novelty}
\vspace{-2mm}
In order to analyze and prove that {\ttfamily DNCB} yields good regret guarantees, we introduced this notion of \emph{forced exploration}. Roughly speaking, this is the regret incurred due to exploration of an agent, when the higher ranked agents are exploring. This extra regret is a consequence of the serial-dictatorship (which we define in Section~\ref{sec:setup}), whereby agents can incur collision and do not get any reward. Although agents in the stationary setting also incur forced exploration, its effect is bounded since every agent can eventually guarantee that the best arm can be learnt. However, in an asynchronously varying environment, bounding this term is non-trivial. We circumvent this by decomposing the forced exploration of an agent recursively; an agent ranked $r$ \emph{effectively explores} if either its own environment is fluctuating and thus hard to identify its best arm, or if the agent ranked $r-1$ is \emph{effectively exploring}. We leverage this to recursively bound the regret of agent ranked $r$ as a function of agent ranked $r-1$. Unravelling this recursion yields the final regret.
\vspace{-2.5mm}
\subsubsection{Experiments}
\vspace{-2mm}
We empirically validate our algorithms to demonstrate that it {\em (i)} is simple to implement and {\em (ii)} the results match the theoretical insights, such as agents incurring additional regret due to forced explorations. 

\noindent One criticism to our model is that the agents are aware of the parameter $\delta$, which is used in the algorithm. We however argue that even in the presence of this known parameter, designing decentralized algorithms is challenging and requires several technical novelty. Parameter free algorithms that do not require any knowledge of $\delta$ is unknown even for the single agent bandit problem \cite{krishnamurthy2021slowly}. Designing parameter free algorithms in the multi-agent case is more challenging and is left to future work.
\vspace{-3mm}
\section{Related work}
\vspace{-1.5mm}
\paragraph{Bandits and Matching Markets}
\vspace{-2mm}
Bandits and matching markets have received a lot of attention lately, owing to both their mathematical non-triviality and the enormous practical impact they hold. Regret minimization in matching markets was first introduced in \cite{liu2020competing} which studied the much simpler problem of stationary markets under a centralized scheme, where a central entity matches agents with arms at each time. They showed that under this policy, a learning algorithm can get per-agent regret scaling as $\mathcal{O}(\log(T))$. Subsequently, \cite{ucbd3} studied the decentralized version of the problem under the serial dictatorship and proposed the {\ttfamily UCB-D3} algorithm that achieved $\mathcal{O}(\log(T))$ per-agent regret. Subsequently, \cite{liu2021bandit} proposed {\ttfamily CA-UCB}, a fully decentralized algorithm that could achieve $\mathcal{O}(\log^2(T))$ per-agent regret in the general decentralized stationary markets. 
Matching markets has been an active area of study in combinatorics and theoretical computer science due to the algebraic structures they present \cite{pittel1989average, roth1990random, knuth1997stable}. 
However, these works consider the equilibrium structure and not the learning dynamics induced when participants do not know their preferences.
\vspace{-1.5mm}
\paragraph{Non-Stationary Bandits}
The framework on non stationary bandits were introduced in \cite{whittle1988restless} with restless bandits. There has been a line of interesting work in this domain--for example in \cite{garivier2011upper,auer2019,liu2018change} the abruptly changing setup is analyzed, and change point based detection methods were employed. Furthermore, in \cite{besbes2014stochastic}, a total variation budgeted setting is considered, where the total amount of (temporal) variation is known. On the other hand, \cite{wei2018abruptly,krishnamurthy2021slowly} focuses on the \emph{smoothly} varying non-stationary environment. Note that \cite{wei2018abruptly} modify the sliding window UCB algorithm of \cite{garivier2011upper} and employ windows of growing size. On the other hand, very recently \cite{krishnamurthy2021slowly} analyzed the \emph{smoothly} varying framework by designing windows of dynamic length and test for optimality within a sliding window.
\vspace{-2.5mm}
\paragraph{Notation:} For a positive integer $r$, we denote the set $\{1,2,\ldots,r\}$ by $[r]$. Moreover, For $2$ integers, $a,b$, the notation $a \% b$  implies the remainder (modulo) operation.

\vspace{-3mm}
\section{Problem Setup}
\label{sec:setup}
\vspace{-3mm}
We consider the standard setup with $N$ agents and $k$ arms, with $k \geq N$. At time $t$, every agent $j \in [N]$ has a ranking of the arms, which is dictated by the arm means $\{\mu_{j,\ell,t}\}_{j \in [N], \ell \in [k]} $. On the other hand, it is assumed that the agents are ranked homogeneously for all the arms, and the ranking is known to the arms. This is called the \emph{serial dictatorship} model, is a well studied model in the market economy (see \cite{abdulkadirouglu1998random, ucbd3}), and without loss of generality, it is assumed that the rank of agent $j \in [N]$ is $j$. We say agent $j$ is matched to arm $\ell$ at time $t$, if agent $j$ pulls and receives (non zero) reward from arm $\ell$.  Our goal here is to find the unique stable matching (uniqueness ensured by the serial dictatorship model) between the agents and the arm side in a non-stationary (dynamic) environment. We consider the smooth varying framework of \cite{wei2018abruptly,krishnamurthy2021slowly} to model the non-stationary, which assumes $|\mu_{j,\ell,t+1} -\mu_{j,\ell,t}| \leq \delta$ for all $t,j,k$, and the maximum drift is $\delta$.

We write $\ell_*^{(1,t)}$ as the arm preferred by the the Agent ranked $1$ at time $t$, i.e., $\ell_*^{(1,t)} = \mathrm{argmax}_{\ell \in [k]} \mu_{1,\ell,t}$. Similarly, for Agent ranked $j$, the preferred arm is given by $\ell_*^{(j,t)} = \mathrm{argmax}_{\ell \in [k] \setminus \{\ell_*^{(1,t)},.,\ell_*^{(j-1,t)} \} } \mu_{j,\ell,t}$. So, we see that $(1,\ell_*^{(1,t)})$ forms a stable match, and so does $(j,\ell_*^{(j,t)})$ for $2 \leq j \leq N$. Let $L^{(j)}(t)$ be the arm played by an algorithm $\mathbb{A}$. The regret of agent $j$ playing algorithm $\mathbb{A}$ upto time $T$ is given by $
    R_j = \sum_{t=1}^T \mathbb{E} [\mu_{j, \ell_*^{(j,t)}, t} - \mu_{j, L^{(j)}(t), t} \mathbf{1}_{M_{L^{(j)}(t)} = j}],
$ where $M(.)$ indicates whether arm $L^{(j)}$ is matched.

\vspace{-2.5mm}
\section{Warm-up: \texttt{DNCB} with $2$ agents}
\label{sec:algo}
\begin{algorithm}[t!]
   \caption{\texttt{DNCB} with $N=2$}
   \label{algo:two}
\begin{algorithmic}
   \STATE {\bfseries Input:} Horizon $T$, drift limit $\delta$
   \STATE Initialize set of tuples $S_1 = \phi$, $S_2^{(j)} = \phi,$ $\forall$ $j \in [k]$, Initialize episode index $i_1 \gets 1$, $i_2 \gets 1$
   \STATE \texttt{RANK ESTIMATION()}
   \FOR{$t=1,2,\ldots,T$}
   \STATE \underline{\textbf{\texttt{Pull-Arm by Agent 1:}}}
   \STATE if $S_1 = \phi$, pull arm $ t \mathbin{\%} k$ (round robbin) (\texttt{Explore}); if $\exists (x,s) \in S_1$ s.t. $s>t$, play arm $x$ (\texttt{Exploit})
   \STATE \underline{\texttt{Test by Agent 1:}}\\
   \textbf{if} $\exists$ arm $a$ and $\Tilde{\lambda}$ s.t. $a>_{\Tilde{\lambda}} b$ for $b \in [k]\setminus \{a\}$ \textbf{then}, 
   $\Lambda_{i_1} \gets t - s_{i_1}$, $\text{buf}_1 = \frac{2}{\delta}\sqrt{\frac{k\log T}{\Lambda_{i_1}}}$
   \STATE \quad \textbf{if}  $\text{buf}_1 > \Lambda_{i_1}$, $S_1 \gets S_1 \cup \{(a,s_{i_1} + \text{buf}_1)\}$,  Updates black-board with $(a,s_{i_1} + \text{buf}_1)$
   \STATE \quad \textbf{else} $i_1 \gets i_1 +1$, 
   $s_{i_1} \gets t$, \quad $i_2 \gets i_2+1, t_{i_2} \gets t$
   \STATE \underline{\texttt{Release arm by Agent 1:}}
   \STATE \quad \textbf{if} $\exists(x,s)\in S_1 : s \leq t$ \textbf{then} $S_1 \gets S_1 \setminus (x,s)$, release arm $x$ 
    \STATE \quad \qquad $i_1 \gets i_1+1, s_{i_1} \gets t$, $i_2 \gets i_2+1, t_{i_2} \gets t$
   \STATE \underline{\textbf{\texttt{Pull Arm by Agent 2:}}}
   \STATE \textbf{Case I:} if \texttt{Agent 1} is not committed, pull arm $t +1 \mathbin{\%} k $ (round robbin on $[k]$) (\texttt{Explore ALL})
   \STATE \textbf{Case II:} if \texttt{Agent 1} is committed to arm $j\in [k]$ and $S_2^{(j)} = \phi$, pull $t \% (k-1)$-th smallest arm id in $[k]\setminus \{j\}$ (round robbin on $[k]\setminus \{j\}$) (\texttt{Explore-j}) 
   \STATE \textbf{Case III:} if \texttt{Agent 1} is committed to arm $j$, and $\exists (x,s) \in S_2^{(j)}$ s.t. $s > t$, play arm $x$ (\texttt{Exploit})
   \textbf{if}  $z_t(2) \neq z_{t-1}(2)$ \textbf{then}, $i_2 \gets i_2 +1$, $t_{i_2} \gets t$
   \STATE \underline{\texttt{Test by Agent 2:}}
   \FOR{$j \in [k]$ s.t. \texttt{Agent 2} is in \texttt{Explore-j} or \texttt{Explore ALL}}
   \IF{$\exists$ arm $a \in [k]\setminus \{j\}$ and $\Tilde{\lambda}$ s.t. $a>_{\Tilde{\lambda}} b$ for $b \in [k]\setminus \{a,j\}$}
   \STATE $\tau_{i_2}^{(j)} \gets t - t_{i_2}$, $\text{buf}_2 =  \frac{2}{\delta}\sqrt{{|\mathcal{E}_{i_2}|}\log T/\tau_{i_2}^{(j)}} $;  define $\Bar{t}_{i_2} = \min\{t_{i_2} + \text{buf}_2, s_{i_1 +1}\}$
   \STATE \quad \textbf{if} $\text{buf}_2 > \tau_{i_2}^{(j)}$, Update $S_2^{(j)} \gets S_2^{(j)} \cup \{(a,\Bar{t}_{i_2})\}$ 
   \ENDIF
   \ENDFOR
   \STATE \underline{\texttt{Release arms for Agent 2:}}
   \FOR{$j \in[k]$}
   \STATE \quad \textbf{if} $\exists(x,s)\in S_2^{(j)} : s \geq t$ \textbf{then},  $S_2^{(j)} \gets S_2^{(j)} \setminus (x,s)$
   \ENDFOR
   \ENDFOR
\end{algorithmic}
\end{algorithm}
\vspace{-2.5mm}
We now propose and analyze the algorithm, Decentralized  Non-stationary Competing Bandits (\texttt{DNCB}) to handle the competitive nature of a market framework under a smoothly varying non-stationary environment. To understand the algorithm better, we first present the setup with $2$ agents and $k$ arms, and then in Section~\ref{sec:gen}, we generalize this to $N$ agents.

We consider $N=2$, since it is the simplest non-trivial setup to gain intuition about the complexity of the competitive nature of \texttt{DNCB} algorithm. Without loss of generality, assume that agent $r$ has rank $r$, where $r \in \{1,2\}$. So, in the above setup, Agent 1 is the highest ranked agent. 
\vspace{-1.5mm}
\paragraph{Black Board model:} Moreover, to begin with and for simplicity, we  assume a black-board model, and later in Section~\ref{sec:board}, remove the necessity of this black board. We emphasize that black-board model of communication is quite standard in centralized multi-agent systems, with applications in game theory, distributed learning and auction applications \cite{awerbuch2008competitive,buccapatnam2015information,agarwal2012information}. Through this black-board, the agents can communicate to one other. This is equivalent to broadcasting in the centralized framework.

The learning algorithm is presented in Algorithm~\ref{algo:two}. The algorithm runs over several episodes indexed by $i_1$ and $i_2$ for Agent 1 and 2 respectively.
\vspace{-1.5mm}
\paragraph{\texttt{RANK ESTIMATION ():}} We let both agents pull arm 1 in the first time slot. Agent 1, will see a (non-zero) reward, and hence estimates its rank to be 1. The other agent, will see a 0 reward, so it estimates its rank as 2. 
\vspace{-2mm}
\paragraph{Agent 1:} Since Agent 1 is highest ranked agent, it does not face any collision. It plays the well-known and standard Successive Elimination (SE) type algorithm (see \cite{slivkins2019introduction}). As mentioned in Section~\ref{sec:intro}, we use a variation of \texttt{SnoozeIT} algorithm of \cite{krishnamurthy2021slowly} with $k$ arms. Specifically, it (a) first explores to identify if there is a \emph{best} arm  and (b) if it finds a best arm, it commits to that for some amount of time. Note that with non-stationary environment, Agent 1 needs to repeat this procedure over time. In Figure~\ref{fig:algo}, we consider one episode of Agent 1, where the yellow segments indicate the exploration time, and at the end of that, the purple segment indicates the commit (exploitation) (to say arm $i^*$) time. Furthermore, when Agent 1 commits, it writes the arm on which it is committing and the duration of the commit to the black-board, so that Agent 2 can accordingly choose actions from a restricted set of arms to avoid collision. Note that, there is no competition here, and the (interesting) market aspect is absent. 

We now define an optimality test via which Agent 1 (and 2) decides to commit. Let $\hat{\mu}_{a,t}(\Tilde{w})$  denote the empirical reward mean of arm $a$ at time $t$, based on its last $\Tilde{w}$ pulls.
\begin{definition}($(\Tilde{\lambda}, \mathcal{A})$-optimality)
At time $t$, an arm $a$ is said to be $\Tilde{\lambda}$ optimal with respect to set $\mathcal{A}$, if $\hat{\mu}_{a,t}(\Tilde{w}) > \hat{\mu}_{b,t}(\Tilde{w}) + 4r(\Tilde{w}) -\delta$, for all $b \in \mathcal{A}\setminus \{a\}$, where $\Tilde{w} = \frac{c_1 \log T}{\Tilde{\lambda}^2}$, and $r(\Tilde{w}) = \sqrt{\frac{2\log T}{\Tilde{w}}}$.
\label{defn:lambda_better}
\end{definition}
\vspace{-1mm}
Since Agent 1 faces no competition, $\mathcal{A} = [k]$ (the set of all arms), but $\mathcal{A}$ will be different for Agent 2, as we will see shortly. In Algorithm~\ref{algo:two}, we denote $\Lambda_{i_1}$ as the duration of the exploration period before the test succeeds (with $\mathcal{A}=[k])$ at episode $i_1$, and we use $\{s_{i_1}\}_{i_1=1,2,..}$ to denote the starting of epochs. After the test, the agent exploits the best arm for $(\text{buf}_1 - \Lambda_{i_1})$ time, and then releases it. We define the set $S_1$ to determine whether Agent 1 should commit or continue exploring.

\begin{figure}[t!]
\subfloat[][]{
    \includegraphics[width=0.48\linewidth]{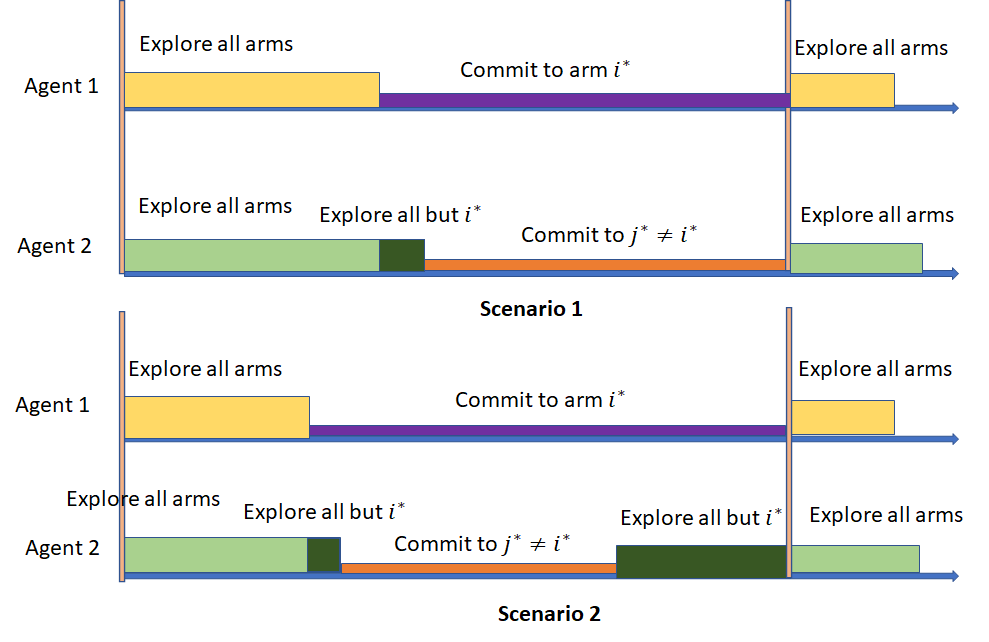}
    \label{Action of Agents 1 and 2 in a matching markets}
    }
    \subfloat[][]{
    \includegraphics[width=0.48\linewidth]{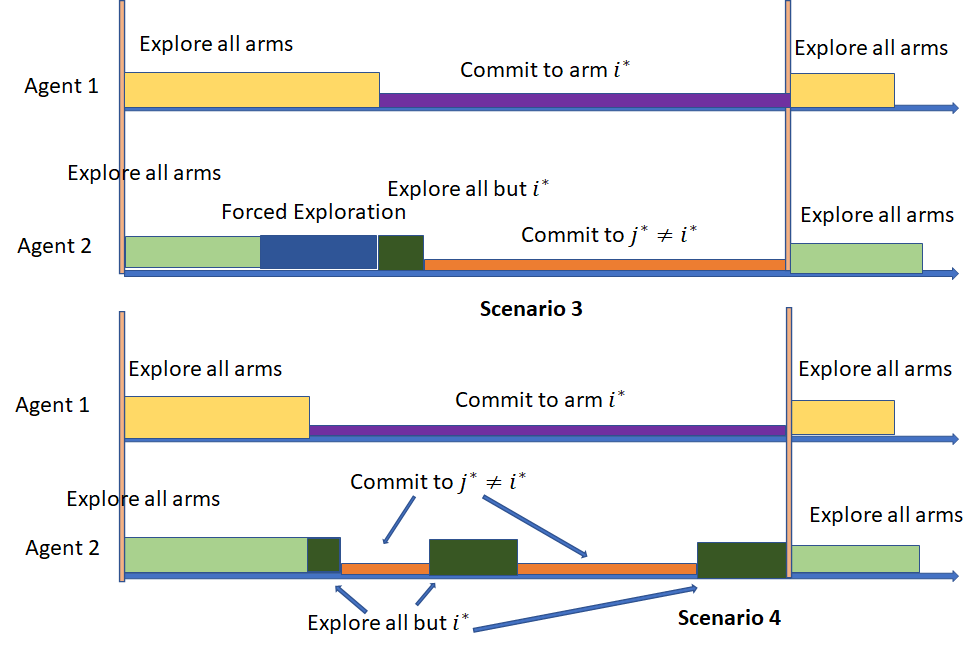}
    \label{Action of Agents 1 and 2 in a matching markets}
    }
\caption{Action of Agents 1 and 2 in a matching markets}
\label{fig:algo}
\vskip -0.2in
\end{figure}
\vspace{-2mm}
\paragraph{Agent 2:} The actions of Agent 2 borne out the competition (market) aspect of the problem, as seen in Figure~\ref{fig:algo}. We now explain its different phases:

\texttt{Explore ALL:} Here, Agent 2 explores all the arms, i.e., plays in a round robbin fashion within the arm-set $[k]$. This is shown in light green in Figure~\ref{fig:algo}. This happens when Agent 1 is also exploring and has not committed yet.

\texttt{Explore-j:} This is shown in dark-green in Figure~\ref{fig:algo}. Here, Agent 2 explores within the set $[k] \setminus \{j\}$. In the figure, $j = i^*$. This is done to avoid collision for Agent 2, since we know that when Agent 1 commits to arm $i^*$, Agent 2 will get $0$ reward while pulling $i^*$, and hence it is in its best interest to explore all but $i^*$.

\texttt{Forced Exploration:} Consider Scenario 3 of Figure~\ref{fig:algo}(b). Here, Agent 2 has decided to commit on an arm before Agent 1. However, it cannot start to exploit since Agent 1 is still exploring. Otherwise, it will periodically face collisions (and get $0$ reward, hence incurring linear regret in this duration). This is the additional exploration faced by Agent 2, which we term as \emph{forced exploration} (shown in blue in Figure~\ref{fig:algo}(b)). In Theorem~\ref{thm:two}, we characterize the regret stems from this forced exploration.

\texttt{Exploit:} Observe that Agent 2 gets only gets to commit when Agent 1 has committed already, on the arm set $[k]\setminus \{i^*\}$. There is another caveat here. We also restrict Agent 2 to end its exploitation as soon as the exploitation of Agent 1 ends. The reasoning is same---Agent 1  starts exploring right after its exploitation and Agent 2 must release the arm it was exploiting to avoid collision. Note that this also results in higher regret of Agent 2, as it does not get to \emph{fully exploit} the arm it was committed to.

In Algorithm~\ref{algo:two}, we denote $\{t_{i_2}\}_{i_2 = 1,2,..}$ as the time instances where an epoch starts for Agent 2. We denote by $\mathcal{E}_{i_2} \subseteq \{1,\cdots,k\}$ to be the set of arms from which agent $2$ plays in phase $i_2$. Observe from  Figure~\ref{fig:algo} that in any given phase of Agent $2$, the set of arms it plays from is fixed. Moreover, we use the notation $z_t(2)$ to denote the state of Agent 2, and as explained in Algorithm~\ref{algo:two}, $
z_t(2) \in \{\texttt{Explore ALL}, \texttt{Explore}-j, \texttt{Exploit}(x) \}$, 
where the terms are explained above. Furthermore, we define $\tau_{i_2}^{(j)}$ as the duration of the exploration period before the $(\Tilde{\lambda},\mathcal{A})$ test succeeds with $\mathcal{A}=[k]\setminus \{j\}$ in epoch $i_2$.
We introduce $\Bar{t}_{i_2}$ to ensure that the exploitation of Agent 2 expires as soon as the exploitation ends for Agent 1.

\textbf{\textit{Saving extra exploration:}}  Note that Agent 2 continues to test for an optimal arm even when Agent 1 is exploring. It might seem to be wasteful at first since it cannot commit immediately. However, Agent 2 constructs the sets $S_2^{(j)}$, which denote the exploitation period of Agent 2, without arm $j$ in the system. This is useful because, as soon as Agent 1 commits to arm $j$ and $S_2^{(j)}$ is non-empty, Agent 2 gets to commit leveraging this test. This saves extra exploration for Agent 2 and hence reduces regret.

\vspace{-2mm}
\subsection{Problem Complexity---Dynamic Gap}
\vspace{-2mm}
We define the (dynamic) gap, denoted by $\{\lambda_t [r] \}_{t=1,2,..}$ for agent $r$, which determines how complex the problem is. This is expressed as an average gap over a local window. 
\begin{definition}
For $\mathcal{C}\subseteq [k]$, we define the dynamic gap on a dominated set $\mathcal{C}$ as,
\small
\vspace{-2mm}
\begin{align*}
    \lambda_t^\mathcal{C}[r]\!\! = \!\!\max_{\lambda \in [0,1]} \bigg \lbrace \!\! \min_{\substack{a,b \in [k] \setminus \mathcal{C} \\ a \neq b}}  \hspace{2mm} \frac{1}{w_k(\lambda)} | \sum_{t' = s}^t \! \mu_{a,r,t'} \!- \!\mu_{b,r,t'} | \geq \lambda \bigg \rbrace,
    \vspace{-2mm}
\end{align*}
\normalsize
and if such a $\lambda$ does not exist, we set $\lambda_t = c_1\frac{ (k-|\mathcal{C}|) \log T}{t}$. Here, $s=t-w_k(\lambda)+1$, and $w_k(\lambda) = \frac{c_0 \log T}{\lambda^2}$. For shorthand, if $\mathcal{C} = \phi$, we denote $\lambda_t^\phi[r] = \lambda_t [r]$. Here $c_1$ and $c_0$ are universal constants. \end{definition}
\begin{remark}
The \emph{dominated} dynamic gap is a strict generalization of the usual window based average gap used in non-stationary bandits. We introduce a dominated set $\mathcal{C}$, for the competitive market setting, since the actions of lower ranked agents are dominated by that of higher ranked ones.
\end{remark}
\vspace{-3mm}
\subsection{Regret Guarantee}
\vspace{-1mm}
\begin{theorem}[2 Agent \texttt{DNCB}]
\label{thm:two}
Suppose we run Algorithm~\ref{algo:two} with $2$ Agents upto horizon $T$ with drift $\delta$. Then the expected regret for Agent 1 is $R_1 \leq C  \sum_{\ell=1}^m \frac{1}{\lambda_{\min,\ell}[1]} \ k \log T$ and for Agent 2 is
\small
\begin{align*}
     R_2 \leq C_1 \sum_{\ell=1}^m \bigg \lbrace ( \frac{1}{\lambda_{\min,\ell}[2]} + \frac{1}{(\lambda_{\min,\ell}[1])^2}) k \log T  + \lceil (\frac{k}{k-1})^{1/3} \rceil (\frac{1}{\min_{a \in [k]} \lambda^{\{a\}}_{\min,\ell}[2]} ) (k-1) \log T \bigg \rbrace,
\end{align*}
\normalsize
where the horizon $T$ is divided into $m$ blocks, each having length at most $\min\{ c \, \delta^{-2/3} k^{1/3} \log^{1/3} T, T \}$. Here  $\lambda_{\min,\ell}[r] = \min_{t \in \ell\text{-th block}} \lambda_t[r] $ and $\lambda^{\{a\}}_{\min,\ell}[r] = \min_{t \in \ell\text{-th block}} \lambda^{\{a\}}_t[r]$ denote the dynamic gap of the problem over an entire $\ell$-th block.
\end{theorem}
\vspace{-2mm}
\paragraph{Discussion:} \textbf{\textit{Regret of Agent 1 matches \cite{krishnamurthy2021slowly}}}: Observe that the regret of Agent 1 matches exactly to \texttt{Snooze-IT}. Since Agent 1 faces no collision, we were able to recover the regret of \texttt{Snooze-IT}.

\textbf{\textit{Regret of Agent 2}}: The regret of Agent 2 has 3 components. The first term, $\frac{k \log T}{\lambda_{\min,\ell}[2]}$ comes from the \texttt{Explore-ALL}. In this phase, Agent 2 explores all arms and the regret is similar to Agent 1. 

The second term in regret, $[\frac{1}{\lambda_{\min,\ell}[1]}]^2 k \log T$ originates from the \texttt{Forced Exploration} of Agent 2. Note that this depends on the complexity (gap) of Agent 1. This validates our intuition, because, when Agent 1's environment is complex, it takes more exploration for Agent 1, and as a result Agent 2 faces additional forced exploration. This is a manifestation of the market structure, since the regret of Agent 2 is influenced by that of higher ranked agent.

The third term in the regret expression comes from \texttt{Explore-j} phase, where Agent 1 is committed on arm $j$. Observe that here, the dominated gap naturally comes into the picture. The pre-factor of $[k/(k-1)]^{1/3}$ appears for the following reason. We design the blocks in such a way that each block contains at most 2 phases of Agent 1. Moreover, we show that the number of epochs for Agent 2 in one exploitation phase of Agent 1 is at most $2 \lceil [k/(k-1)]^{1/3} \rceil$.

\textbf{\textit{Regret matches to UCB-D3 of \cite{ucbd3} in stationary setup}}: We compare the regret of \texttt{DNCB} with that of the non-stationary UCB-D3 of \cite{ucbd3}. In the stationary environment ($\delta=0$), the definition of gap is invariant with time. For Agent 2, from \citep[Corollary 2]{ucbd3}, we obtain the regret to be $\mathcal{O}[\frac{1}{\rho^2}(k-1) \log T]$, where $\rho$ is the stationary dominated gap. Note that this is exactly same as Theorem~\ref{thm:two} (except for a mildly worse dependence on $k$). Hence, we recover the order-wise optimal regret in the stationary setting.

\vspace{-3.5mm}
\section{\texttt{DNCB} Algorithm with $N$ competing agents}
\label{sec:gen}
\vspace{-3mm}
In this section, we extend \texttt{DNCB} for $N$ agents. We stick to the setup where the $r$-th Agent is ranked $r$ and focus on the learning algorithm of the $r$-th agent. Let us fix some notation. We denote $C_t(i) \in \{\phi,1,\ldots,k\}$ as the arm committed by Agent $i$ at time $t$. For Agent $r \in [N]$, we (sequentially) define $\mathcal{C}_t(r) = \{C_t(1),\ldots,C_t(r-1)\}$ as the set of committed (dominated) arms by agents ranked higher that $r$. The learning scheme is presented in Algorithm~\ref{algo:gen}

\begin{algorithm}[t!]
   \caption{\texttt{DNCB} for $r$-th Agent}
   \label{algo:gen}
\begin{algorithmic}
   \STATE {\bfseries Input:} Horizon $T$, drift limit $\delta$
   \STATE Initialize $S_r^{(\Omega)}= \phi$ for all $\Omega \subseteq [k]$, and $|\Omega| \leq r-1$,  Initialize $i_r \gets 1$,   $\mathcal{C}_1(r) = \phi$
   \STATE \texttt{RANK ESTIMATION()}
   \FOR{$t=1,2,\ldots,T$}
   \STATE \underline{\textbf{\texttt{Update State $z_t(r)$:}}} \\
   \textbf{if} $|\mathcal{C}_t(r)| < r-1$ \textbf{then},  $z_t(r) \gets \texttt{Explore}-\mathcal{C}_t(r)$\\
   \textbf{else}
   \textbf{if} $\exists (x,s) \in S_r^{(\mathcal{C}_t(r))}$ s.t. $s>t$, $z_t(r) \gets \texttt{Exploit}(x)$ \textbf{else} $z_t(r) \gets \texttt{Explore}-\mathcal{C}_t(r)$ \\
   \textbf{if} $z_t(r) \neq z_{t-1}(r)$ \textbf{then} $i_r \gets i_r +1, \,\,\, t_{i_r} \gets t$
   \STATE \underline{\textbf{\texttt{Pull-Arm by Agent r:}}}
   \STATE \textbf{Case I:} if $z_t = \texttt{Explore}-\mathcal{C}_t(r)$, pull $t+(r-|\mathcal{C}_t(r)|-1) \% (k - |\mathcal{C}_t(r)|)$ smallest arm in $[k] \setminus \mathcal{C}_t(r)$ (Play round robbin with $[k] \setminus \mathcal{C}_t(r)$ \STATE \textbf{Case II:} if $z_t = \texttt{Exploit}(x)$, then play arm $x$.
   \STATE \underline{\textbf{\texttt{Test by Agent $r$:}}}
   \FOR{$\Omega \subseteq [k]$ s.t. $|\Omega| =r-1$ and Agent $r$ is in \texttt{Explore}-$\mathcal{C}_t(r)$ }
   \IF {$\exists a \in [k] \setminus \Omega$ and $\Tilde{\lambda}$ s.t. $a >_{\Tilde{\lambda}} b$ for $b \in [k] \setminus \{\Omega \cup \{a\}\}$}
   \STATE $\tau_{i_r}^{(\Omega)} \gets t - t_{i_r}$, $\text{buf}_r = \frac{2}{\delta}\sqrt{(k-|\mathcal{C}_t(r)|)\frac{\log T}{\tau_{i_r}^{(\Omega)}}}$,  define $\Bar{t}_{i_r} = \min \{t_{i_r} + \text{buf}_{i_r}, t_{i_{r-1} +1} \} $
   \STATE \quad \textbf{if} $\text{buf}_{i_r} > \tau_{i_r}$, \textbf{then} $S_r^{\Omega} \gets S_r^{\Omega} \cup \{(a,\Bar{t}_{i_r})\}$, \textbf{else} $i_{i_r} \gets i_{r} +1,  t_{i_r} \gets t$
   \ENDIF
   \ENDFOR
   \STATE \underline{\textbf{\texttt{Update Black Board:}}}
   \STATE Updates $\mathcal{C}_{t+1} = \{x\in [k]: \exists s > t+1,  \,\, \text{and} \,\, \exists j \leq r-1 \,\,\, s.t. \,\,\, (x,s,j) \,\, \text{exists on board} \}
   $
   \STATE if $\exists (x,s)$ s.t. $s\geq t+1$, write $(x,\Bar{t}_{i_t},r)$ on the board
   \ENDFOR
\end{algorithmic}
\end{algorithm}
\textbf{\texttt{RANK ESTIMATION()}}  We start with the rank estimation which takes $N-1$ time steps. At $t=1$, all agents pull arm $1$. Subsequently, for $t\in [2, N-1]$, agents, never matched to any arms play arm $t$, and the agents who were matched to an arm, continues to play the matched arm. By inductive reasoning, one can observe that this collision routine ensures that all agents know their own rank.

We denote $\{t_{i_r}\}_{i_r =1,2,..}$ as the start epochs for Agent $r$. To identify the state of Agent $r$ we define 
$z_t(r) = \{ \texttt{Explore}-\mathcal{C}, \texttt{Exploit}(x)\}$, where in $\texttt{Explore}-\mathcal{C}$, the $r$-th agent plays in a round round robbin on the set of $[k] \setminus \mathcal{C}$ arms, and in $\texttt{Exploit}(x)$ it pulls arm indexed by $x$.

At any time $t$, Agent $r$ first looks at the black-board, and using the information, it constructs a dominated set $\mathcal{C}_t(r)$, which contains all the committed arms from Agents 1 to $r-1$. Based on $\mathcal{C}_t(r)$, Agent $r$ updates $z_t(r)$ to reflect whether it is in \texttt{Explore}-$\mathcal{C}_t(r)$ phase, or in the exploit phase. In particular, Agent $r$ gets to commit on an arm in $[k] \setminus \mathcal{C}_t(r)$, if all the higher ranked agents have already committed, i.e., $|\mathcal{C}_t(r)| = r-1$. A new phase is spawned for Agent $r$ if either the dominated set $\mathcal{C}_t(r) \neq \mathcal{C}_t(r-1)$ changes, or its own phase ends. Both this cases are captured by $z_t(r)$, and hence, based on whether $z_t(r)$ changes or not, Agent $r$ decides to start a new phase.

The test procedure of Agent $r$ is similar to that of Agent 2, with a difference that Agent $r$ tests in the arms in the subset $[k] \setminus \mathcal{C}_t(r)$, and hence the buffer length is accordingly designed. We also need to ensure that Agent $r$ ends its exploitation phase when any higher ranked agent starts exploring. This is ensured by defining $\Bar{t}_{i_r}$.

\textbf{\textit{Saving extra exploration:}} Furthermore, Agent $r$ constructs the sets $S_r^\Omega$ for all $\Omega \subseteq [k]$ with $|\Omega| = r-1$. As explained in the 2 agent case, this saves extra exploration for Agent $r$, because if the statistical test succeeds on an arm $j \in [k] \setminus \mathcal{C}_t(r)$, and there exists $\Omega$, with $|\Omega|=r-1$ such that $S_r^\Omega$ is non-empty, Agent $r$ immediately commits to arm $j$.
\vspace{-3mm}
\subsection{Regret Guarantee}
\vspace{-2mm}
We characterize the regret of $r$-th agent, with $r \geq 2$. Note that the regret of Agent $1$ will be identical as Theorem~\ref{thm:two}, since it faces no competition and hence no collision.

The regret of $r$-th Agent will depend on the dynamic gap of Agents 1 to $r-1$, and hence to ease, we define
\begin{align*}
   \Delta_t[r] = \!\! \min_{\mathcal{C}\in [k], |\mathcal{C}| \leq r-2} \lambda^\mathcal{C}_{t}[r], \hspace{2mm} \Tilde{\Delta}_t[r] = \!\! \min_{\mathcal{C}\in [k], |\mathcal{C}| \leq r-1} \lambda^\mathcal{C}_{t}[r]
\end{align*}
Note that the definitions of $\Delta_t[r]$ and $\Tilde{\Delta}_t[r]$ are generalizations of $\lambda^\mathcal{C}_t[r]$, with further restrictions on the dominated set $\mathcal{C}$. With this, we have the following result:
\begin{theorem}[$N$ agent \texttt{DNCB}]
\label{thm:gen}
Suppose we run Algorithm~\ref{algo:gen} for $N$ agents with $\delta$ drift. The regret for $r$-th ranked agent is given by
\small
\vspace{-1mm}
\begin{align*}
    R_r &\leq C \sum_{\ell=1}^m ( \frac{k}{k-r+2} )^{1/3} \Bigg \lbrace  \left[ \frac{ k\log T}{\Delta_{\min,\ell}[r]} + \frac{ k\log T}{\Delta^2_{\min,\ell}[r-1]} \right]  + \left\lceil \left ( \frac{k-r+2}{k-r+1} \right)^{1/3} \right \rceil \left( \frac{1}{\Tilde{\Delta}_{\min,\ell}[r]} \right) (k-r+1) \log T \Bigg \rbrace,
\end{align*}
\normalsize
where we divide the horizon $T$ in $m$ blocks,  having at most $\min\{ c \, \delta^{-2/3} k^{1/3} \log^{1/3} T, T \}$ length. Here $\Delta_{\min,\ell}[r] = \min_{t \in \ell \,\, block} \Delta_t[r]$ and $\Tilde{\Delta}_{\min,\ell} = \min_{t \in \ell \,\, block} \Tilde{\Delta}_t[r] $ denotes the gap for $\ell$-th block.
\end{theorem}
\textbf{Discussion:} The performance of Agent $r$ depends crucially on Agent $r-1$, and based on whether Agent $r-1$ is exploring or exploiting, the regret depends on the higher ranked $r-2$ agents. Hence, the dynamic gap depending on both $r-1$ and $r-2$ sneaks in the regret expression via $\Delta_t[r]$ and $\Tilde{\Delta}_t[r]$.

\textbf{\textit{Special case, $r=2$:}}
When $r=2$ in Theorem~\ref{thm:gen}, we exactly get back the regret of Agent 2 in Theorem~\ref{thm:two}. So, for Agent 2, there is no additional cost for extending \texttt{DNCB} to $N$ agents.

\textbf{\textit{Different terms:}} Similar to the 2 agent case, the first term presents the regret from exploration of Agent $r$, when Agent $r-1$ is exploring. Hence, the size of the dominated set is at most $r-2$. Similarly, the second term corresponds to the \emph{forced exploration} of Agent $r$. Note that this depends on how complex the system of Agent $r-1$ is. Furthermore, the third term corresponds to the regret when Agent $r$ has committed, and hence the size of dominated set is at most $r-1$. These are characterized by $\Delta_t[r]$ and $\Tilde{\Delta}_t[r]$.

Note that we characterize the regret of Agent $r$ by focusing on one phase of Agent $r-1$ (similar to the 2 agent case), and we show that the number of epochs of Agent $r-1$ in one block is at most $4 \lceil [k/(k-r+2)]^{1/3} \rceil$, which causes the multiplicative pre-factor. Note that with $r=2$, the factor is absent, since the blocks are designed to contain at most 2 epochs of Agent 1.

\textbf{\textit{Matches UCB-D3 of \cite{ucbd3} in stationary setup}:} Note that, in the stationary setup ($\delta=0$), the regret expression in Theorem~\ref{thm:gen} matches to that of UCB-D3 (except a mildly weak dependence on $k$), which is shown to be order optimal. So, \texttt{DNCB} recovers the optimal regret in the stationary case.
\begin{figure*}[t!]
\centering
    \subfloat[][]{
    \includegraphics[width=0.185\linewidth]{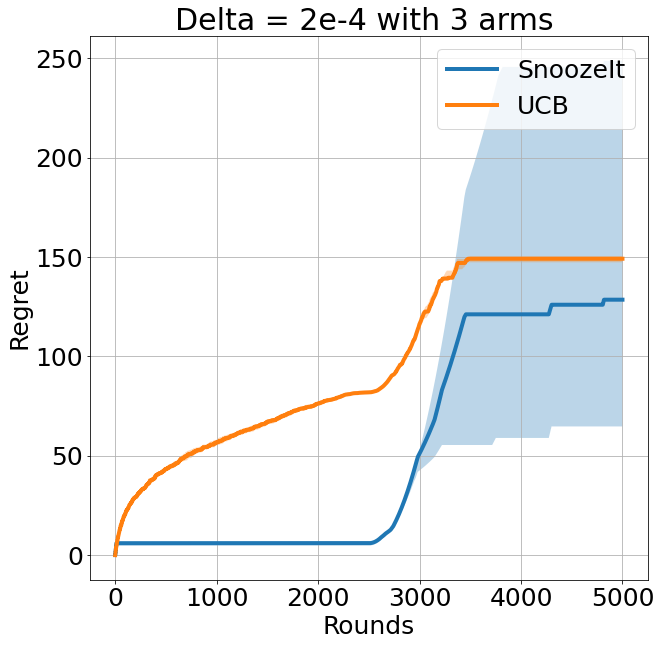}
    \label{fig:single_agent}
    }
    \subfloat[][]{
    \includegraphics[width=0.185\linewidth]{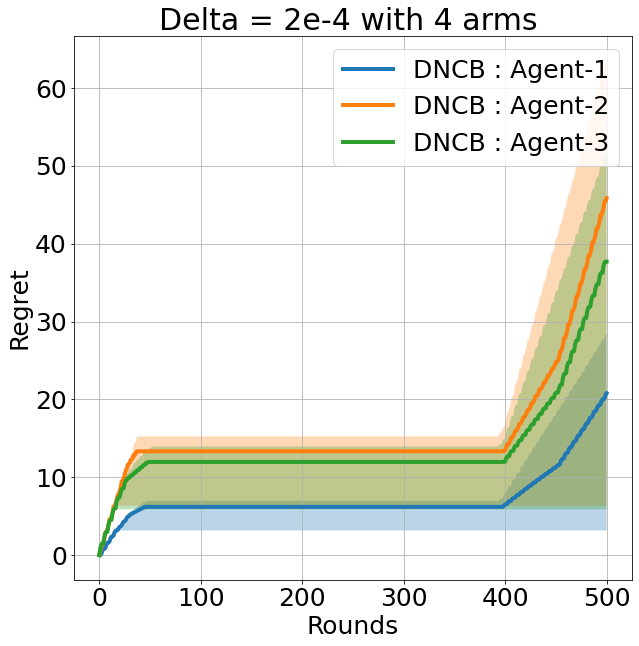}
    \label{fig:multi_agent}
    }
    \subfloat[][]{
    \includegraphics[width=0.185\linewidth]{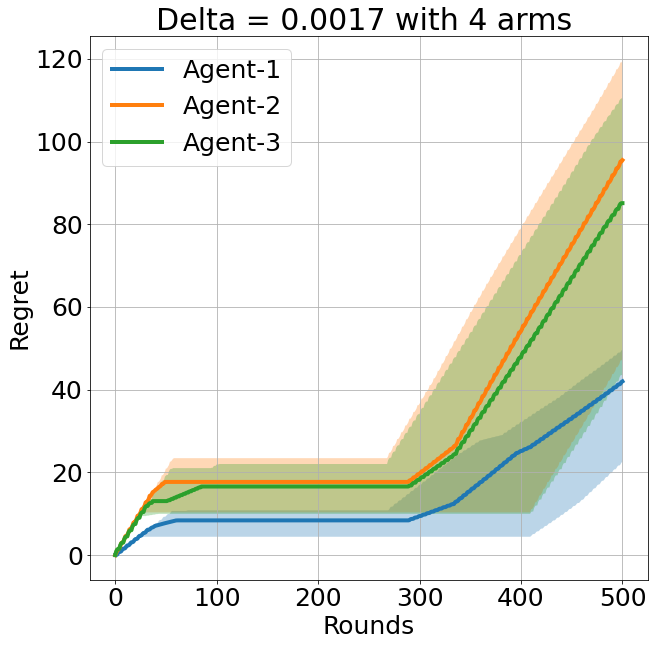}
    \label{fig:multi_agent_2}
    }
    \subfloat[][]{
    \includegraphics[width=0.185\linewidth]{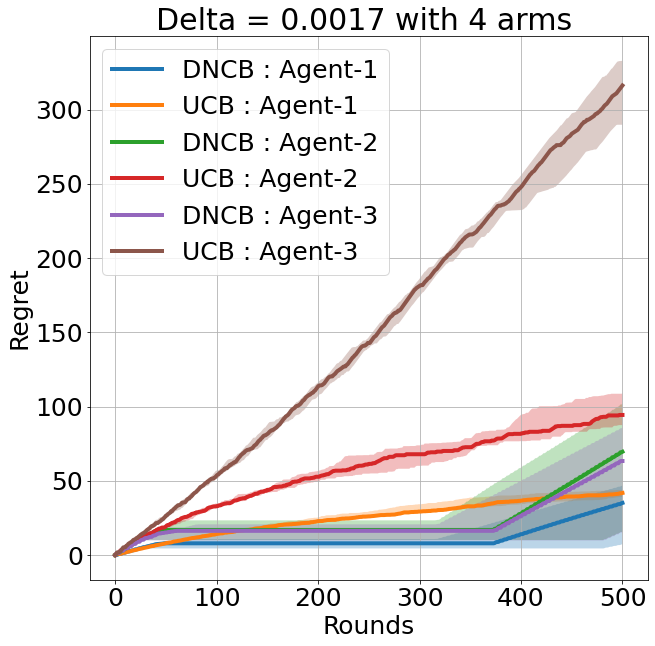}
    \label{fig:dncb_d3}
    }
    \subfloat[][]{
    \includegraphics[width=0.185\linewidth]{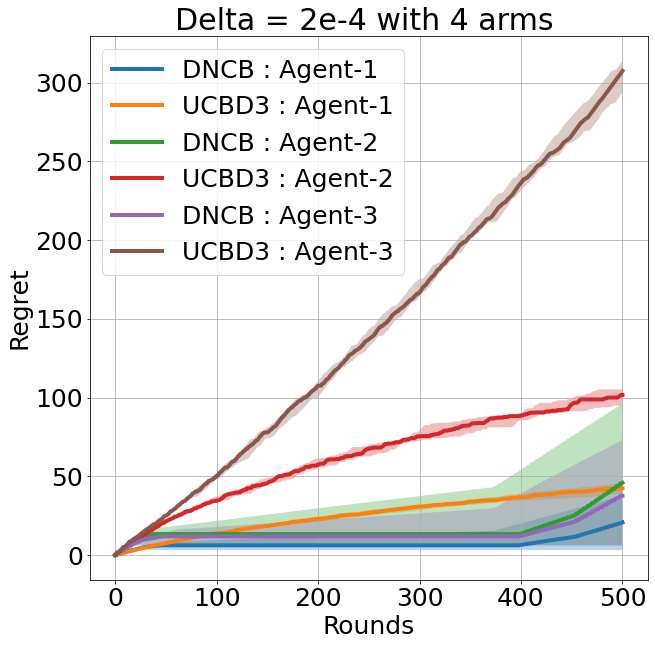}
    \label{fig:dncb_d3_2}
    }
    \caption{\label{fig:simualtions}In $(a)$, we compare SnoozeIT and UCB with $k=3$. In $(b)$ and $(c)$ DNCB on a system with 3 agents and 4 arms is simulated, and the same systems are comapred with UCB-D3 in $(d)$ and $(e)$}
\end{figure*}
\vspace{-4mm}
\section{\texttt{DNCB} without Black Board}
\label{sec:board}
\vspace{-3mm}
Upto now, we present \texttt{DNCB} with a black board, via which the agents communicate among themselves. In this section, we remove this, and obtain the same information via collision. We emphasize that without the black board, the learning algorithm is \emph{completely} decentralized.
\vspace{-3mm}
\subsubsection{Special case: Black board removal with $N=2$}
\vspace{-2mm}
In the presence of the black board, Agent 2 knows whether Agent 1 is exploring (on all $[k]$ arms) or committed to an particular arm.
The same information can be gathered from a collision. 
Agent 2 maintains a latent variable $z_t$, which indicates whether Agent 1 is in \texttt{Explore} or \texttt{Exploit} phase. At the beginning, $z_0 \gets \texttt{Explore}$.

If at round $t$, Agent 2 faces a collision on arm $j$, one of two things can happen---(a) Agent 1 has (ended exploring and) committed to arm $j$ or (b) Agent 1 (has ended its exploitation and) is exploring. This is true from the design of \texttt{DNCB}. After a collision, Agent 2 looks at $z_{t-1}$. If $z_{t-1} = \texttt{Explore}$, then case (a) has happened and if $z_{t-1} = \texttt{Exploit}$, then case (b) has happened. So, just toggling the variable $z_t$ is enough for Agent 2 to keep track of Agent 1. It is easy to see that, from the round robbin structure of exploration, that after Agent's 1 phase changes, it may take upto $k$ time steps for a collision to take place.
\begin{lemma}[Regret Guarantee]
\label{rem:slow}
Suppose $\delta' \leq \frac{c}{k} \delta$, for a constant $c <1$. Then, for a $\delta'$ shifted system,  \texttt{DNCB} without blackboard satisfies the regret guarantees identical to that of Theorem~\ref{thm:two} (with $\delta$). 
\end{lemma}
\vspace{-2mm}
For a $\delta$ shifted system, upto time $k$, the maximum total shift is $k \delta$, and hence with $\delta'$, we ensure that the system remains stationary in these $k$ time steps. We emphasize that \texttt{DNCB} is an asynchronous algorithm, and hence, without black board, we require an even slower varying system to maintain stationarity.

\vspace{-2.5mm}
\subsubsection{Black board removal with $N$ agents}
\vspace{-2mm}
\texttt{DNCB} is an asynchronous algorithm, and hence establishing coordination between agents is quite non-trivial. In previous works, such as \cite{ucbd3}, the learning includes a fixed set of time slots for communication among agents. This coordinated communication can not be done for \texttt{DNCB}, since the phases start and end at random times. Hence, to handle this problem, we consider a slightly stronger reward model.

\textbf{Reward model:} To ease communication across agents, we assume that in case of collision, the reward is given to the highest ranked agent, and all the remaining agent gets zero reward, as well as the index of the agent who gets the (non-zero) reward. We remark that this side information is not impratical in applications like college admissions, job markets etc., and this exact reward model is also studied in \cite{liu2021bandit}.

Under this new reward model, Agent $r$ maintains a set of latent variables, $Q_t[s]$ for all $s \in [r-1]$, where $Q_t[s] \in \{\texttt{Explore}, \texttt{Exploit}\}$. If at time $t$, Agent $r$ experiences a collision, and the reward goes to an Agent $r'$, with $r'<r$, then Agent $r$ toggles $Q_t[r']$. In this way, after a collision on arm $j$, Agent $r$ knows that either Agent $r'$ has committed on arm $j$ or it is exploring on a set of arms including $j$---and based on $Q_{t-1}[r']$, Agent $r$ knows which event has happened exactly. From the round robbin nature of exploration, detecting this may take at most $k$ steps.
\begin{lemma}[Regret Guarantee]
Suppose $\delta' \leq \frac{c}{k} \delta$, for a constant $c<1$. For a $\delta'$ shifted system,  \texttt{DNCB} without blackboard satisfies the regret guarantees identical to that of Theorem~\ref{thm:gen} (with $\delta$).
\label{remark:remove_black_board}
\end{lemma}
\vspace{-2mm}
The above remark holds under the modified and stronger reward model. Design of an efficient coordination protocol in an asynchronous system is left to future work. 

\vspace{-4mm}
\section{Simulations}
\vspace{-3.5mm}

In Figure~\ref{fig:simualtions}, we demonstrate the effectiveness of DNCB on synthetic data. In Figure \ref{fig:single_agent}, we observe that when the environment is varying, {\ttfamily SnoozeIt} outperforms vanilla UCB algorithm. 
In Figures \ref{fig:multi_agent} and \ref{fig:multi_agent_2}, we simulate DNCB on two instances with different arm-means and dynamics. We can observe from these plots that the exploitation time of agent $2$ is strictly within that of agent $1$, and similarly that of agent $3$ is strictly within that of agent $2$. This visually captures the notion of forced explorations, where an agent can only exploit arms, when all higher ranked agents are themselves exploiting arms. 

In Figures \ref{fig:dncb_d3} and \ref{fig:dncb_d3_2}, we compare the performance of DNCB with that of UCB-D3 in a dynamic environment. We find that although the performance of agent $1$ is similar in the two systems, the performance of the lower ranked agents are much superior in DNCB compared to UCB-D3. This shows that DNCB is sensitive to the potential variations in arm-means and helps all agents adapt faster compared to UCB-D3 which is designed assuming the environment is stationary. 
The exact details on the experiment setup is given in the Supplementary materials in Section \ref{sec:additional_experiments}.




\vspace{-4mm}
\section{Conclusion and open problems}
\vspace{-3mm}
We introduced the problem of decentralized, online learning in two-sided markets when the underlying preferences vary smoothly over time. This paper however leaves several intriguing open problems: (a) to understand whether the assumption of known $\delta$ be relaxed; (b) extend the dynamic framework to general markets beyond serial dictatorship; and (c) to consider other forms of non-stationary such as piece-wise constant markets or variations with a total budget constraint.


\bibliographystyle{abbrvnat}
\bibliography{markets_bandits_non_stationary}

 \clearpage
 \appendix
 \begin{center}
    \textbf{\Large{Supplementary Material for ``Competing Bandits in Non-Stationary Matching Markets''} }
\end{center}
\vspace{3mm}

\section{Related Works on Non-Stationary Bandits}
The framework on non stationary bandits were introduced in \cite{whittle1988restless} in the framework of restless bandits, and later improved by \cite{Slivkins_adaptingto}. There has been a line of interesting work in this domain--for example in \cite{garivier2011upper,auer2019,liu2018change} the abruptly changing or switching setup is analyzed, where the arm distributions are piecewise stationary and an abrupt change may happen from time to time. In particular \cite{liu2018change} proposes a change point based detection algorithm to identify whether an arm distribution has changes of not in a piecewise stationary environment. Furthermore, in \cite{besbes2014stochastic}, a total variation budgeted setting is considered, where the total amount of (temporal) variation is known, but the change may happen, either smoothly or abruptly.

Moreover, in the above-mentioned total variation budget based non-stationary framework, an adaptive algorithm, that does not require the knowledge of the drift parameter is obtained in \cite{karnin2016multi} for the standard bandit problem and later extended to \cite{luo2018efficient} for the contextual bandit setup. 

On the other hand, there are a different line of research that focuses on the \emph{smoothly} varying non-stationary environment, in contrast to the above mentioned abrupt or total budgeted setup, for example see \cite{wei2018abruptly,krishnamurthy2021slowly}. Note that \cite{wei2018abruptly} modify the sliding window UCB algorithm of \cite{garivier2011upper} and employ windows of growing size. On the other hand, very recently \cite{krishnamurthy2021slowly} analyzed the \emph{smoothly} varying framework by designing windows of dynamic length and test for optimality within a sliding window. The algorithm of \cite{krishnamurthy2021slowly}, namely Snooze-IT, is an asynchronous algorithm that works on repeated Explore and Commit (ETC) type principle where the explore and commit times are random. 

In this paper, we work with the smoothly varying non-stationary framework of \cite{krishnamurthy2021slowly}. We choose this algorithm because of its simplicity, and the dynamics and competition that comes out of a market framework is better understood in such a sliding window based Explore and Commit type algorithm. In general, we believe that our basic principle can be adapted to any sliding window based algorithm in a non-stationary environment.

\section{Proof of Theorem~\ref{thm:two} }
\label{sec:thm_one}
{\color{black}

\subsection{Technical Preliminaries} 

As is standard in formalizing bandit processes \cite{lattimore2020bandit}, we assume that the random process lies in a probability space endowed with a collection of independent and identically distributed random variables $(U_{i,j}[t])_{i \in [N], j\in[k] t \geq 1}$. For each $i \in [N]$ and $j \in [k]$, and $k\geq 1$, the random variables $(U_i[k])$ is distributed as the $0$ mean, unit variance Gaussian random variable\footnote{Our analysis can be extended verbatim to any sub-gaussian distribution}. With this description, the realized reward by agent $i \in [N]$, when it matches with arm $j \in [k]$ for the $k^{th}$ time at time-index $t$ is given by $U_{i,j}[k]+ \mu_{i,j,t}$. In this description, the set of arm-means $(\mu_{i,j,t})_{i\in[N],j\in[k],t\in[T]}$ are fixed non-random parameters.

\begin{definition}[Good Event]
\label{defn:good_event}
\begin{align*}
    \mathcal{E} := \bigcap_{i=1}^N \mathcal{E}^{(i)},
\end{align*}
where
\begin{align*}
    \mathcal{E}^{(i)} := \bigg\{ \forall t \in [T], \forall j \in [k],  \forall w \leq t, \bigg| \frac{1}{w}\sum_{s=0}^w U_{i,j}[t-s] \bigg| \leq r(w) \bigg\},
\end{align*}
here $r(w):= \sqrt{\frac{8 \log(T)}{w}}$.
\end{definition}

In words, the event $\mathcal{E}$ is the one in which every contiguous sequence of i.i.d. random variables is `well-behaved'. The event $\mathcal{E}^{(1)}$ is identical to the good-event specified for the single agent case in \cite{krishnamurthy2021slowly}. Standard concentration inequalities give that this occurs with high probability which we record in the proposition below.

\begin{proposition}
\begin{align*}
    \mathbb{P}[\mathcal{E}] \geq 1 - \frac{2Nk}{T^2}.
\end{align*}
\label{prop:good_event_high_prob}
\end{proposition}
\begin{proof}
Fix a $t \in [T], i \in [N],  j \in [k]$ and  $w \leq t$. Classical sub-gaussian inequality gives that
\begin{align*}
    \mathbb{P}\left[ \bigg| \frac{1}{w}\sum_{s=0}^{w-1} U_{i,j}[t-s] \bigg| > r(w) \right] &\leq 2 \exp \left( -\frac{1}{2} w r(w)^2 \right),\\
    &= \frac{2}{T^4}.
\end{align*}
Now, taking an union bound over $t, i, j$ and $w$ gives that 
\begin{align*}
    \mathbb{P}[\mathcal{E}^{\complement}] \leq \frac{2Nk}{T^2}.
\end{align*}
\end{proof}

The definition of the good event is useful due to the following result. 

\begin{lemma}[No regret in the exploit-phase]
If the good event $\mathcal{E}$ in Definition \ref{defn:good_event} holds, then every agent in every exploit phase will incur $0$ regret.
\label{lem:snoozed_arms_are_sub_optimal}
\end{lemma}

\begin{proof}
We first prove the result for agent ranked $1$. For any phase $i_1$ of Agent $1$, denote by time $t=g_{i_1}$ to be the time-instant at which an arm $a$ and $\lambda > 0$ is identified that satisfies $a >_{\lambda} b$ for all $b\in[k]\setminus \{a\}$. In words, time $g_{i_1}$ is the time when the statistical test by Agent $1$ succeeds. Recall from the notations in the algorithm that $\Lambda_{i_1} := g_{i_1}-s_{i_1}$. 

Suppose in a phase $i_1$, agent $1$ exploits an arm $a \in [k]$ one or more rounds. Notationally, this is from times $[g_{i_1}+1, t_{i_1+1}]$. We  will show that {\em (i)} there exists a minimum gap $\lambda > 0$, such that at time $g_{i_1+1}$, for all arms $a^{'} \in [k]\setminus\{a\}$, the mean of arm $a$ exceeds $a^{'}$ by a certain margin, and {\em (ii)} in the duration $[t_{i_1}+\tau_{i_1}, t_{i_1+1}]$ is set such that the chosen arm $a$ continues to be optimal in the entire EXPLOIT phase. The first claim is formalized below. 

\begin{claim}
Under the good event $\mathcal{E}$, there exists a time $t^{'} \in [s_{i_1}, g_{i_1}]$, such that for all arms $b \in [k]\setminus\{a\}$, $\mu_{a,t^{'}} - \mu_{b,t^{'}} \geq  4\sqrt{\frac{\log(T)}{\Lambda_{i_1}}}-k\delta$. 
\label{claim_1}
\end{claim}
\begin{proof}
The statistical test succeeded at time $g_{i_1}$, i.e., there exists a $\lambda > 0$ such that $a >_{\lambda} b$, for all $b \in [k]\setminus \{a\}$. By Definition \ref{defn:lambda_better}, the window size $w := k\lceil \frac{c_1 \log(T)}{\lambda^2} \rceil$. Since the test succeeds at time $t = g_{i_1}$, clearly $w \leq \Lambda_{i_1}$. 

In order to describe the proof, we set some notations. For every arm $a \in [k]$, denote by the set of times $\mathcal{L}_{i_1}^{(a)}:= \{l_1^{(a)}, \cdots, l_{w/k}^{(a)}\}$ to be the $w/k$ times arm $a$ was played in the time-interval $[g_{i_1}-w, g_{i_1}]$. These times are random variables \textemdash however conditioned on $g_{i_1}$, these are deterministic since in the Explore phase of Algorithms \ref{algo:two} and \ref{algo:gen}, agents explore the arms in a round-robin fashion from arms indexed the smallest to the largest. Denote by $\mu_{a,t}(w) := \frac{k}{w}\sum_{s \in \mathcal{L}_{i_1}^{(a)}}\mu_{a,s}$. For every arm $a \in [k]$, denote by the random index $l_a$ to be the number of times arm $a$ has been played in the past, before time $g_{i_1} - w$. 

Since the statistical test succeeds at time $t = g_{i_1}$, we have from Definition \ref{defn:lambda_better}

\begin{align*}
    \mu_{a,t}(w) + \frac{k}{w}\sum_{s=l_a}^{l_a+w/k}U_a[s] > \mu_{b,t}(w) + \frac{k}{w}\sum_{s=l_b}^{l_b+w/k}U_b[s] + 4 r(w/k) - \delta.
\end{align*}
Re-arranging and using the definition of the Good event, we have 
\begin{align*}
    \mu_{a,t}(w) - \mu_{b,t}(w) &> 4r(w/k) - \delta +  \frac{1}{w}\sum_{s=l_b}^{l_b+w/2}U_b[l] -  \frac{1}{w}\sum_{s=l_a}^{l_a+w/2}U_a[l], \\
    &\geq 2r(w/k) - \delta.
\end{align*}
where the second inequality stems from the definition of the good event. Now, since the drift is bounded by $\delta$, we have that 
\begin{align*}
     \mu_{a,t}(w) - \mu_{b,t}(w) &\leq \frac{1}{w}\sum_{s=0}^{w-1}(\mu_{a,t-s} - \mu_{b,t-s}) + (k-1)\delta.
\end{align*}
Combining the preceeding two displays, we get that 
\begin{align*}
    \frac{1}{w}\sum_{s=0}^{w-1}(\mu_{a,t-s} - \mu_{b,t-s}) &> 2r(w/k) - k\delta, \\
    &\geq 2 \sqrt{\frac{4 \log(T)}{\Lambda_{i_1}}} - k\delta = 4\sqrt{\frac{\log(T)}{\Lambda_{i_1}}} - k\delta. 
\end{align*}
The second inequality follows from the fact that the window size $w \leq \Lambda_{i_1}$ is smaller than the explore duration of phase $i_1$. Now, since the average gap exceeds a bound, it implies that there exists at-least one $t^{'} \in [s_{i_1},g_{i_1}]$ such that $\mu_{a,t^{'}} - \mu_{b,t^{'}} > 4\sqrt{\frac{\log(T)}{\Lambda_{i_1}}} - k\delta$.

\end{proof}
Now, since the drift at each time-step in each arm is at-most $\delta$, arm $a$ will remain optimal compared to arm $b$ at-least in the time-interval $[t^{'}, t^{'}+ \frac{2}{\delta}(4\sqrt{\frac{\log(T)}{\Lambda_{i_1}}} - k\delta)]$, i.e., arm $a$ is optimal compared to arm $b$ in the duration $[t^{'}, t^{'} + \frac{4}{\delta}\sqrt{\frac{4\log(T)}{\Lambda_{i_1}}} - k]$. Since $t^{'} \geq s_{i_1}$, and  from Algorithms \ref{algo:two} and \ref{algo:gen} the definition of Buffer is $buf := \frac{4}{\delta}\sqrt{\frac{4\log(T)}{\Lambda_{i_1}}} - k$, arm $a$ is superior to arm $b$ in the exploit duration of phase $i_1$. Now, since arm $b$ was arbitrary, this implies that Agent $1$ will incur no regret during the exploit phase of $i_1$.

For the general case, we will prove by induction. Suppose the induction hypothesis that all agent ranked $1$ through to $r-1$ are incurring $0$ regret in an exploit phase. Notice from the description of Algorithm \ref{algo:gen} that agent ranked $r$ can potentially go into an exploit phase if and only if all agents ranked $1$ through $r-1$ are in an exploit phase. Additionally, the base case of the induction hypothesis is what we established in the preceding paragraph where agent ranked $1$ incurs $0$ regret in the exploit phase. Under this induction hypothesis, we will now argue that agent ranked $r$ will also incur $0$ regret in the corresponding exploit phase. 

We make one observation based on the serial-dictatorship structure. If all agents ranked $1$ through $r-1$ are in {\em (i)} Exploit phase and {\em (ii)} are incurring $0$ regret, then the stable match optimal arm for agent ranked $r$ is to play the arm with the highest mean among those arms not being exploited by agents ranked $1$ through $r-1$. This is a simple consequence of the definition of stable match (c.f. Section \ref{sec:setup}). Thus, it suffices to argue that when agent ranked $r$ commits, it commits to the optimal arm. We use identical arguments as for agent ranked $1$ to show that.
\begin{claim}
If at time $t$, for a given $\Omega \subset [k]$ with $|\Omega| = r-1$, the statistical test succeeds with arm $a \in [k]$, then  there exists a time $t^{'} \in [s_{i_r}, t]]$, such that for all arms $b \in \Omega \setminus \{a\}$, $\mu_{a,t^{'}} - \mu_{b,t^{'}} \geq  4\sqrt{\frac{\log(T)}{\Lambda_{i_1}}}-k\delta$. 
\end{claim}
The proof follows identical arguments as that of Claim \ref{claim_1} by using the observation that $r(w)$ is a decreasing function of $w$. 
\end{proof}
This result was shown for the special case of $N=1$ and $k=2$ in \cite{krishnamurthy2021slowly} Lemma $2$. Lemma \ref{lem:snoozed_arms_are_sub_optimal} generalizes that to the multi-agent and multi-arm setting. 

}
\subsection{Other Notations used in the proof}
\label{subsec:proof_notations}

In order to improve readability, we collect all the notations used in the course of the proof.

We now prove the regret of both Agents 1 and 2 for Algorithm~\ref{algo:two}. Note that, Agent 1 just plays the Snooze-IT algorithm of \cite{krishnamurthy2021slowly}, and hence we borrow the techniques developed there to obtain the regret of Agent 1.

More interestingly, in this section, we provide a full characterization of the regret of Agent 2. Note that since Agent 2 plays on a restrictive or dominated set of arms, dictated by Agent 1, it encounters additional regret. In the description of Algorithm~\ref{algo:two}, we pointed out the scenarios where Agent 2 is forced to (a) either explore or (b) to stop exploiting. Here, we obtain a regret upper-bound from these forced exploration-exploitation.

To better understand the algorithm, let use focus on a particular phase of Agent 1, say the $i_1$-th epoch. We use the same notation defined in Algorithm~\ref{algo:two}. So, $s_{i_1}$ denotes that start-time of epoch $i_1$ ans $s_{i_1 +1}$ denotes the end of epoch $i_1$. The exploration duration before committing to an arm is $\Lambda_{i_1}$, and so the exploitation phase starts at $s_{i_1} + \Lambda_{i_1}$. Similarly, the length of exploitation is $s_{i_1 +1} - \Lambda_{i_1}$. Let us also assume that the committed arm of Agent 1 in this phase is $i^*$.

Since Agent 1 plays Snooze-IT, during the exploitation phase, it incurs no regret during the exploitation phase from Lemma \ref{lem:snoozed_arms_are_sub_optimal}, and from \cite{krishnamurthy2021slowly}, the (expected) regret of Agent 1 in $i_1$-th phase is
\vspace{2mm}
\begin{align*}
    R_1(i_1) \leq \mathcal{O}(\sqrt{\Lambda_{i_1} \, \log T})
\end{align*}
\vspace{2mm}
Technically, the lemmas of \cite{krishnamurthy2021slowly} are under a good event, which is identical to the good event definition in Definition $\mathcal{E}^{(1)}$ \ref{defn:good_event}. We now look at the behavior of Agent 2, while Agent 1 is in phase $i_1$. As shown in Figure~\ref{fig:algo}, there can be multiple phases of Agent 2 inside one phase of Agent 1, and hence let us assume that at the beginning of epoch $i_1$, the phase number of Agent 2, given by $i_2 = n_{i_1}$, and by the end of phase $i_1$, we have $i_2 = n_{i_1} + N_{i_1}$.

\subsection{Regret of Agent 2 during the exploration period of the $i_1$th phase of Agent 1}
In this phase, which lasts for $\Lambda_{i_1}$ rounds, we characterize the regret of Agent 2. For this, let us define $\tau_{n_{i_1}}$ as the duration, starting from $s_{i_1}$ it takes for Agent 2 to commit to an arm unconditionally. This means that in the absence of competition, starting from $s_{i_1}$, Agent 2 would take $\tau_{n_{i_1}}$ to commit to an arm by exploring all the arms. We have 2 cases:

\textit{Case I ($\Lambda_{i_1} \leq \tau_{n_{i_1}}$):} In this case, since Agent 1 commits first, the regret of Agent 2,  is given by $\mathcal{O}(\sqrt{\Lambda_{i_1} \log T})$. In this case, Agent 2 is not forced to explore.

\textit{Case II ($\Lambda_{i_1} \geq \tau_{n_{i_1}}$):} In this case, Agent 2 incurs a regret of $\mathcal{O}(\sqrt{\tau_{n_{i_1} \log T})}$ plus some additional regret owing to force exploration. The forced exploration comes from the fact that in this case, although Agent 2 has enough information to commit, it still explores because Agent 1 has not committed yet, and the commitment of Agent 2 will cause periodic collisions for Agent 2.

\subsubsection{Forced Exploration} We now characterize the regret of Agent 2 form forced exploration. Note that Agent 2 is forced to explore at time $t$ if:
\begin{enumerate}
    \item Agent 1 is exploring, and
    \item At time $t$, $S_2^{(j_t)}$ is non-empty, where $j_t \in [k]$ is the  arm played by Agent 1.
\end{enumerate}
Let us understand this in a bit more detail. If $S^{(j_t)}_2$ is non-empty, it implies that without the presence of competition, Agent 2 would have played arm $j_t$. This comes from the definition of $S_2^{(.)}$. Now, when Agent 1 is playing that arm, it implies a forced exploration on Agent 2. We can write down the above forced exploration term as the following
\vspace{2mm}
\begin{align*}
    \textsf{Forced Exploration} = \sum_{t=s_{i_1}}^{s_{i_1}+ \Lambda_{i_1}} \sum_{j=1}^k \mathbf{1}(j_t = j) \, \mathbf{1}\left( \tau_{n_{i_1}}^{(j)} < t \right),
\end{align*}
\vspace{2mm}
where $\tau_{n_{i_1}}^{(j)}$ is defined as the duration of the exploration period before the $(\Tilde{\lambda},\mathcal{A})$ test succeeds with $\mathcal{A}=[k]\setminus \{j\}$ in epoch $i_2$, when Agent 2 is in state \texttt{Explore ALL}.

Combining this two, the regret of Agent 2 during the exploration phase of Agent 1 is given by
\vspace{2mm}
\begin{align*}
   \mathcal{O}\left[ \mathbf{1} (\text{Case-I}) \sqrt{\Lambda_{i_1} \, \log T} + \mathbf{1} (\text{Case -II}) \left(  \sqrt{\tau_{n_{i_1}} \, \log T} + \sum_{t=s_{i_1}}^{s_{i_1}+ \Lambda_{i_1}} \sum_{j=1}^k \mathbf{1}(j_t = j) \, \mathbf{1}\left( \tau_{n_{i_1}}^{(j)} < t \right) \right) \right]
\end{align*}
\vspace{2mm}
\subsection{Regret of Agent 2 during exploitation phase of Agent 1}
Suppose Agent 1 commits to arm $i^*$. In this phase, Agent 2 is forced to play in a restrictive set $[k] \setminus \{i^*\}$. Note that in this phase, several cases may happen:

\textit{Agent 2 is exploiting:} Note that Agent 2 keeps the set $S_2^{(j)}$ for all $j \in [k]$, and if $j\neq i^*$, Agent 2 immediately commits to $j$. Keeping track of such $S_2^{(j)}$ thus ensures that agent 2's exploration after not wasted.

Furthermore, if $S_2^{(j)}$ is empty, for all $j \neq i^*$, Agent 2 will keep accumulating samples, now from a restrictive set $[k] \setminus \{i^*\}$, and may commit to an arm within the set. In both the cases, Lemma \ref{lem:snoozed_arms_are_sub_optimal} gives that the regret is zero.

\textit{Agent 2 is exploring:} Note that inside the exploit phase of Agent 1, Agent 2 basically plays the Snooze-IT algorithm over the arm-set $[k] \setminus \{i^*\}$. Hence, the regret owing to exploitation is given by
\vspace{2mm}
\begin{align*}
    \mathcal{O}\left( \sum_{p = n_{i_1} +1}^{n_{i_1} + N_{i_1}} \sqrt{(k-1)\Tilde{\tau}^{(i^*)}_p \, \log T}  \right),
\end{align*}
\vspace{2mm}
where $N_{i_1}$ is the number of phases of Agent 2 in the current exploitation phase of Agent 1, and $\Tilde{\tau}_{p}^{(i^*)}$ is defined as the duration of the exploration period before the $(\Tilde{\lambda},\mathcal{A})$ test succeeds with $\mathcal{A}=[k]\setminus \{i^*\}$, when Agent 2 is in state \texttt{Explore}-$i^*$.

\subsection{Total Regret of both agents in one phase} 
Putting everything together, the regret of Agent 1 and 2, denoted by $R_1(i_1)$ and $R_2(i_1)$ respectively, during the $i_1$-th phase of Agent 1 is given by
\begin{align*}
    R_1(i_1) \leq \mathcal{O}(k \sqrt{\Lambda_{i_1} \, \log T}),  \,\, \text{and}
\end{align*}
\begin{align*}
    R_2(i_1) &\leq  \mathcal{O} \Bigg[ \underbrace{\mathbf{1} (\text{Case-I}) \sqrt{\Lambda_{i_1} \, \log T}}_{T_1} + \underbrace{\mathbf{1} (\text{Case -II}) \left(  \sqrt{\tau_{n_{i_1}} \, \log T} + \sum_{t=s_{i_1}}^{s_{i_1}+ \Lambda_{i_1}} \sum_{j=1}^k \mathbf{1}(j_t = j) \, \mathbf{1}\left( \tau_{n_{i_1}}^{(j)} < t \right) \right)}_{T_2} \\
    & \qquad \qquad + \underbrace{\left( \sum_{\ell = n_{i_1} +1}^{n_{i_1} + N_{i_1}} \sqrt{(k-1)\Tilde{\tau}^{(i^*)}_\ell \, \log T}  \right)}_{T_3} \Bigg ]
\end{align*}

\subsubsection{Regret for Agent 1 in phase $i_1$:} 
We now bound $\sqrt{\Lambda_{i_1}}$ using Lemma 4 and Lemma 5 ( of \cite{krishnamurthy2021slowly}). In particular, we extend these lemmas to the $k$ arm case, and obtain
\begin{align*}
    \sqrt{\Lambda_{i_1}} \leq \mathcal{O}\left(\frac{1}{\lambda_{g_{i_1}-1}}\right) \sqrt{k \log T},
\end{align*}
where $g_{i_1} = s_{i_1} + \Lambda_{i_1}$ is the time instant where the test succeeds for Agent 1, and $\lambda_t$ denotes the dynamic gap. Hence, we have
\begin{align*}
    R_1(i_1) \leq \mathcal{O}(\sqrt{k \, \Lambda_{i_1} \, \log T}) \leq  \mathcal{O}\left( \frac{k \log T}{\lambda_{g_{i_1}-1}[1]} \right),
\end{align*}

\subsubsection{Regret for Agent 2 in phase $i_1$} 
We now upper bound $T_1,T_2$ and $T_3$ separately. We first consider $T_1$.

We have
\begin{align*}
    T_1 = \mathbf{1} (\text{Case-I}) \sqrt{\Lambda_{i_1} \, \log T} \leq \sqrt{\Lambda_{i_1} \, \log T}, 
\end{align*}
and using the same modified lemma as before, we obtain
\begin{align*}
    T_1 \leq \mathcal{O}\left( \frac{k \log T}{\lambda_{g_{i_1}-1}[1]} \right),
\end{align*}
where $\lambda_{g_{i_1}-1}[1]$ denotes the dynamic gap for player $1$ at time instant $g_{i_1}-1$.

For $T_2$, we have
\begin{align*}
    T_2 & = \mathbf{1} (\text{Case -II}) \left(  \sqrt{\tau_{n_{i_1}} \, \log T} + \sum_{t=s_{i_1}}^{s_{i_1}+ \Lambda_{i_1}} \sum_{j=1}^k \mathbf{1}(j_t = j) \, \mathbf{1}\left( \tau_{n_{i_1}}^{(j)} < t \right) \right) \\
    & \leq \underbrace{\Bigg (  \sqrt{\tau_{n_{i_1}} \, \log T}}_{T_{2,1}} + \underbrace{ \sum_{t=s_{i_1}}^{s_{i_1}+ \Lambda_{i_1}} \sum_{j=1}^k \mathbf{1}(j_t = j) \, \mathbf{1}\left( \tau_{n_{i_1}}^{(j)} < t \right) \Bigg )}_{T_{2,2}}
\end{align*}
The term $T_{2,1}$ can be bounded similar to $\Lambda_{i_1}$. This is the exploitation time of Agent 2 in the Explore all phase. Hence, it can be upper bounded as
\begin{align*}
    T_{2,1} \leq \mathcal{O}\left( \frac{k \log T}{\lambda_{\Tilde{g}_{i_1}-1}[2]} \right),
\end{align*}
where $\Tilde{g}_{i_1} = s_{i_1} + \tau_{n_{i_1}}$ is the time instant where the test succeeds for Agent 2, and$\lambda_{\Tilde{g}_{i_1}-1}[2]$ denotes the dynamic gap for player $2$ at time instant $\Tilde{g}_{i_1}-1$.

Note that during the exploration phase of Agent 1, the arms are being played in a round robbin fashion, and hence
\begin{align*}
    T_{2,2} &\leq \sum_{t=s_{i_1}}^{s_{i_1}+ \Lambda_{i_1}} \sum_{j=1}^k \mathbf{1}(j_t = j) \, \mathbf{1}\left( \tau_{n_{i_1}}^{(j)} < t \right) \leq \sum_{t=s_{i_1}}^{s_{i_1}+ \Lambda_{i_1}} \sum_{j=1}^k \mathbf{1}(j_t = j) \\
    & \leq \sum_{j=1}^k \sum_{t=s_{i_1}}^{s_{i_1}+ \Lambda_{i_1}}  \mathbf{1}(j_t = j) \, \mathbf{1}\left( \tau_{n_{i_1}}^{(j)} < t \right) \\
    & \leq \sum_{j=1}^k \frac{\Lambda_{i_1}}{k} = \Lambda_{i_1}.
\end{align*}
Hence, we have
\begin{align*}
    T_{2,2} \leq \mathcal{O}\left[\left( \frac{1}{\lambda_{g_{i_1}-1}[1]} \right)^2 k \log T \right].
\end{align*}
Combining $T_{2,1}$ and $T_{2,2}$, we have
\begin{align*}
    T_2 \leq \mathcal{O} \left[\left( \frac{k \log T}{\lambda_{\Tilde{g}_{i_1}-1}[2]}  \right) + \left( \frac{1}{\lambda_{g_{i_1}-1}[1]} \right)^2 k \log T  \right]
\end{align*}
\vspace{2mm}

Let us now control $T_3$. Note that during the exploitation phase of Agent 1, Agent 2 only incurs regret while exploring within the set of $[k]\setminus \{i^*\}$, and the regret incurred from that is equivalent to playing a Snooze-IT algorithm on arm-set $[k]\setminus \{i^*\}$. So, using the modified lemma now using on arm set $[k]\setminus \{i^*\}$ with cardinality $k-1$ is given by
\vspace{2mm}
\begin{align*}
    \sqrt{\Tilde{\tau}_p^{(i^*)}} \leq \mathcal{O} \left( \frac{\sqrt{(k-1) \log T}}{\lambda^{(i^*)}_{(g_{n_{i_1},p})-1}[2]}\right),
\end{align*}
\vspace{2mm}
where $g_{n_{i_1},p}$ is the time instant where the test succeeds when Agent 2 is in $p$-th phase. Furthermore, since Agent 2 is not playing arm $i^*$, this regret depends on the dynamic gap excluding arm $i^*$, denoted by $\lambda^{(i^*)}_{(.)}$. Using this, we have
\vspace{2mm}
\begin{align*}
    T_3 = \sum_{p = n_{i_1} +1}^{n_{i_1} + N_{i_1}} \sqrt{(k-1)\Tilde{\tau}^{(i^*)}_p \, \log T}  \leq  \sum_{p = n_{i_1} +1}^{n_{i_1} + N_{i_1}} \left(\frac{(k-1)\log T}{\lambda^{(i^*)}_{(g_{n_{i_1},p})-1}[2]}\right).
\end{align*}
\vspace{2mm}
Combining $T_1,T_2$ and $T_3$, we obtain
\vspace{2mm}
\begin{align*}
    R_2(i_1) &\leq \mathcal{O} \left[ \left( \frac{k \log T}{\lambda_{g_{i_1}-1}[1]} \right) + \left( \frac{k \log T}{\lambda_{\Tilde{g}_{i_1}-1}[2]}  \right) + \left( \frac{1}{\lambda_{g_{i_1}-1}[1]} \right)^2 k \log T +  \sum_{p = n_{i_1} +1}^{n_{i_1} + N_{i_1}} \left(\frac{(k-1)\log T}{\lambda^{(i^*)}_{(g_{n_{i_1},p})-1}[2]}\right) \right] \\
   & \leq  \mathcal{O} \left[ \left( \frac{k \log T}{\lambda_{\Tilde{g}_{i_1}-1}[2]}  \right) + \left( \frac{1}{\lambda_{g_{i_1}-1}[1]} \right)^2 k \log T + \sum_{p = n_{i_1} +1}^{n_{i_1} + N_{i_1}} \left(\frac{(k-1)\log T}{\lambda^{(i^*)}_{(g_{n_{i_1},p})-1}[2]}\right) \right],
\end{align*}
\vspace{2mm}
since $\lambda_t \in [0,1]$. What remains is a bound on $N_{i_1}$.

\subsection{Total Regret upto time $T$}
In the above calculations, we have the regret for the $i_1$-th phase of Agent 1 only. Note that the starting instances of epochs for Agent 1, denoted by $\{s_{i_1}\}_{i_1=1,2,..}$ is random. To handle this issue, the learning epoch is split into several (deterministic) blocks and the total regret guarantee is given over these deterministic splits.

\subsection{Total Regret for Agent 1}
We derive the Lemma 7 of \cite{krishnamurthy2021slowly}, for the case of $k$ arms and obtain that the minimum length of an epoch of Agent 1 is given by $\Omega(\delta^{-2/3} k^{1/3} \log^{1/3} T)$. Motivated by this, we fix the deterministic blocks of length $\delta^{-2/3} k^{1/3} \log^{1/3} T$ so that each block can accommodate at most $2$ phases. Using this, we write the regret of Agent 1 as
\vspace{2mm}
\begin{align*}
    R_1 \leq C \sum_{\ell=1}^m \frac{1}{\lambda_{\min,\ell}[1]} \ k \log T,
\end{align*}
\vspace{2mm}
where $m$ denotes the number of blocks, each having length at most $\min\{ c \, \delta^{-2/3} k^{1/3} \log^{1/3} T, T \}$, and $\lambda_{\min,\ell}[1] = \min_{t \in \ell\text{-th block}} \lambda_t $.

\subsubsection{Total Regret of Agent 2}
Now let us look at Agent 2. Note that in the exploitation phase of Agent 1, Agent 2 either plays Snooze-IT with $k-1$ arms, or uses the optimistic estimates to exploit. In any case, from the point of view of incurring regret, the performance of Agent 2 in the exploitation time of Agent 1 is that of Snooze-IT with $k-1$ arms without competition.

So, using Lemma 7 of \cite{krishnamurthy2021slowly}, the minimum length between 2 epochs of Agent 2 is given by
$$
\Omega \left (\delta^{-2/3} (k-1)^{1/3} \log^{1/3} T \right).$$ 

Note that the since an entire phase of Agent 1, which includes exploration as well as exploitation is lower bounded by $\Omega(\delta^{-2/3} k^{1/3} \log^{1/3} T)$, trivially the exploitation phase is at least, $\Omega(\delta^{-2/3} k^{1/3} \log^{1/3} T)$. Hence the number of epochs played by Agent 2 for during the $i_1$-th phase of Agent 1 is given by
\vspace{2mm}
\begin{align*}
    N_{i_1} \leq 2 \times 2 \times \left \lceil \left(\frac{k}{k-1}\right)^{1/3} \right \rceil.
\end{align*}
\vspace{2mm}
We are now ready to write the total regret of Agent 2 upto time $T$. It is given by
\vspace{2mm}
\begin{align*}
    R_2 &\leq C_1 \sum_{\ell=1}^m \bigg \lbrace \left( \frac{1}{\lambda_{\min,\ell}[2]} \right) k \log T + \left( \frac{1}{\lambda_{\min,\ell}[1]} \right)^2 k \log T \\
    & \qquad + \left \lceil \left(\frac{k}{k-1}\right)^{1/3} \right \rceil \left(\frac{1}{\min_{a \in [k]} \lambda^{(a)}_{\min,\ell}[2]} \right) (k-1) \log T \bigg \rbrace,
\end{align*}
\vspace{2mm}
where the number of blocks is denoted by $m$, each having length at most $\min\{ c \, \delta^{-2/3} k^{1/3} \log^{1/3} T, T \}$, and
$$
\lambda_{\min}[\ell] = \min_{t \in \ell\text{-th block}} \lambda_t. $$

Furthermore, $\lambda^{(a)}_{(.)}$ denotes the dynamic gap in the problem without arm $a$. This concludes the theorem.

\vspace{4mm}
\section{Proof of Theorem~\ref{thm:gen}}
\label{sec:gen_proof}
In this theorem, we consider the generic case of $N$ agents, and we characterize the regret of agent ranked $r$. We consider the learning of Agent $r-1$ as the action of Agent $r$ will be dominated by that. The proof here follows in the same lines as of Theorem~\ref{thm:two}. The problem has an inductive structure, and this proof exploits that. It turns out that without loss of generality, we may only focus on the behavior of $r-1$-th agent; very similar to focusing on the first agent in the previous theorem.

\subsection{Behavior of $r-1$-th ranked Agent}

We consider $1$ epoch of agent $r-1$. From the notation of Algorithm~\ref{algo:gen}, it starts at $t_{i_{r-1}}$, and let the exploration period is $\Lambda_{i_{r-1}}$. Similarly, the exploitation period duration is $t_{i_{r-1}+1} - \Lambda_{i_{r-1}}$.

Note that if $r\geq 3$, the exploration of Agent $r-1$ will be restricted. Let $\mathcal{C}_t(r-1)$ be the set of arms dominated by agents ranked $1$ to $r-2$, i.e., $|\mathcal{C}_t(r-1)| \leq r-2$. With this, the dynamic gap parameter for Agent ranked $r-1$ is given by $\lambda_t^{\mathcal{C}_t(r-1)}[r-1]$. Note that when $\mathcal{C}_t(r-1) = \phi$, Agent $r-1$ will Explore all arms.

\subsubsection{Regret of Agent $r$ in explore phase of Agent $r-1$}
As presented in the previous theorem, we break he regret of Agent $r$, during the exploration and the exploitation phase of agent $r-1$. 

During the exploration phase, Agent $r$ can either Explore all arms, or explore within a restricted set. Recall that $\mathcal{C}_t(r)$ denotes the set of arms dominated by agents ranked higher than Agent $r$. If $\mathcal{C}_t(r) = \phi$, Agent $r$ explores all the arms. Otherwise it will explore the set of arms given by $[k]\setminus \mathcal{C}_t(r)$.

Similar to the 2 agent case, here also, Agent 2 will face \textit{forced exploration}, and the definition is identical to the one in 2 agent case---instead of conditioning on the behavior of Agent 1, here, we condition on the behavior of Agent $r-1$.

Following the same lines, we obtain the regret of Agent $r$ in the exploration phase of Agent $r-1$ is given by
\vspace{2mm}
\begin{align*}
     \mathcal{O} \left[ \left( \frac{(k - |\mathcal{C}_t(r-1)|) \log T}{\lambda_{g_{i_{r-1}}-1}^{\mathcal{C}_t(r-1)}[r-1]} \right) + \left( \frac{(k- |\mathcal{C}_t(r)|) \log T}{\lambda^{\mathcal{C}_t(r)}_{g_{i_r}-1}[r]}  \right) + \left( \frac{1}{\lambda_{g_{i_{r-1}}-1}^{\mathcal{C}_t(r-1)}[r-1]} \right)^2 (k- |\mathcal{C}_t(r-1)|) \log T  \right]
\end{align*}
\vspace{2mm}
where the time instances, $g_{i_{r-1}}$ denote the time the $(\Tilde{\lambda},\mathcal{A})$ test succeeds for Agent $r-1$ with $\mathcal{A}=[k]\setminus \mathcal{C}_t(r-1)$. Similarly, $g_{i_r}$ denote the time $(\Tilde{\lambda},\mathcal{A})$ test succeeds for Agent $r$ with $\mathcal{A}=[k]\setminus \mathcal{C}_t(r)$. Note that $|\mathcal{C}_t(r-1)| \leq r-2$ and $|\mathcal{C}_t(r)| \leq r-2$, since Agent $r-1$ has not committed yet. We upper bound the following as
\vspace{2mm}
\begin{align*}
     \mathcal{O} \left[  \left( \frac{(k- |\mathcal{C}_t(r)|) \log T}{\lambda^{\mathcal{C}_t(r)}_{g_{i_r}-1}[r]}  \right) + \left( \frac{1}{\lambda_{g_{i_{r-1}}-1}^{\mathcal{C}_t(r-1)}[r-1]} \right)^2 (k- |\mathcal{C}_t(r-1)|) \log T  \right]
\end{align*}
\vspace{2mm}
\subsubsection{Regret of Agent $r$ in exploit phase of Agent $r-1$}
Similar to the behavior of Agent $2$, in this case Agent $r$ may be multiple epochs inside an exploration period of Agent $r-1$. 

Note that inside the exploit phase of Agent 1, Agent 2 basically plays the Snooze-IT algorithm over the arm-set $[k] \setminus \mathcal{C}_t(r)$. Hence, the regret owing to exploitation is given by
\vspace{2mm}
\begin{align*}
    \mathcal{O}\left( \sum_{p = i_{r} +1}^{i_r + N_{i_r}} \sqrt{(k-|\mathcal{C}_t(r)|)\Tilde{\tau}^{(\mathcal{C}_t(r))}_p \, \log T}  \right),
\end{align*}
\vspace{2mm}
where $N_{i_r}$ is the number of phases of Agent 2 in the current exploitation phase of Agent 1, and $\Tilde{\tau}_{j}^{(\mathcal{C}_t(r))}$ is defined as the duration of the exploration period before the $(\Tilde{\lambda},\mathcal{A})$ test succeeds with $\mathcal{A}=[k]\setminus \mathcal{C}_t(r)$.

We bound the above as
\vspace{2mm}
\begin{align*}
    \mathcal{O} \left[ \sum_{p = i_{r} +1}^{i_r + N_{i_r}} \left(\frac{(k-|\mathcal{C}_t(r)|)\log T}{\lambda^{\mathcal{C}_t(r)}_{(g_{i_r,p})-1}[2]}\right) \right].
\end{align*}
\vspace{2mm}

We now need to bound $N_{i_r}$. Note that, when Agent $1$ commits, $|\mathcal{C}_t(r-1)| = r-2$. As a consequence, using \citep[Lemma 7]{krishnamurthy2021slowly}, the minimum length of an episode for Agent $r-1$ is $\Omega(\delta^{-2/3} (k-r+2)^{1/3} \log^{1/3} T$. Hence, we have
\begin{align*}
    N_{i_r} \leq 2 \times 2 \times \left\lceil \left ( \frac{k-r+2}{k-r+1} \right)^{1/3} \right \rceil.
\end{align*}
\vspace{2mm}

We now break the learning horizon into deterministic epochs. We use deterministic blocks of fixed length given by $\mathcal{O}(\delta^{-2/3} k^{1/3} \log^{1/3} T)$. Now, within one such block, the number of epochs of Agent $r-1$ is upper bounded by 
\vspace{2mm}
\begin{align*}
    \left \lceil \left( \frac{k}{k-r+2} \right)^{1/3} \right \rceil.
\end{align*}
\vspace{2mm}
Hence, in one such deterministic block the regret of Agent $r$ will be multiplied by the regret in one phase of Agent $r-1$ times the number of phases of Agent $r-1$.

\subsubsection{Regret Expression}
We are now ready to write the expression of regret for Agent $r$. We have
\vspace{2mm}
\begin{align*}
    R_r &\leq C \sum_{\ell=1}^m \Bigg \lbrace \left( \frac{k}{k-r+2} \right)^{1/3} \left[\left( \frac{1}{\displaystyle \min_{\substack{\mathcal{C}\in [k] \\ |\mathcal{C}| \leq r-2}} \lambda^\mathcal{C}_{\min,\ell}[r]} \right) + \left( \frac{1}{\displaystyle \min_{\substack{\mathcal{C}\in [k] \\ |\mathcal{C}| \leq r-2}} \lambda^\mathcal{C}_{\min,\ell}[r-1]} \right)^2 \right] k \,\log T \\
    & \qquad + \left\lceil \left ( \frac{k-r+2}{k-r+1} \right)^{1/3} \right \rceil \left( \frac{1}{\displaystyle \min_{\substack{\mathcal{C}\in [k] \\ |\mathcal{C}| \leq r-1}} \lambda^\mathcal{C}_{\min,\ell}[r]} \right) (k-r+1) \log T \Bigg \rbrace,
\end{align*}
where we now discuss several terms. 
\vspace{2mm}
The term $ \min_{\substack{\mathcal{C}\in [k] \\ |\mathcal{C}| \leq r-2}} \lambda^\mathcal{C}_{\min,\ell}[r]$ denotes the (worst-case) gap, of Agent $r$ on a subset $\mathcal{C}$ of cardinality at most $r-2$. Note that this is an lower bound on the term $\lambda^{\mathcal{C}_t(r)}_{g_{i_r}-1}[r]$. Furthermore, since we do not have a lower bound on $|\mathcal{C}_t(r)|$, we upper bound $k-|\mathcal{C}_t(r)|$ as $k$.

Similarly, the second term comes from forced exploration. The final term also follows from the exploitation of Agent $r$. Here, $ \min_{\substack{\mathcal{C}\in [k] \\ |\mathcal{C}| \leq r-1}} \lambda^\mathcal{C}_{\min,\ell}[r]$ denotes the (worst-case) gap, of Agent $r$ on a subset $\mathcal{C}$ of cardinality at most $r-1$. Note that this is an lower bound on the term $\lambda^{\mathcal{C}_t(r-1)}_{g_{i_r}-1}[r-1]$. This proves the theorem.

\subsection{Proof of Lemma~\ref{rem:slow}}

The proof comes from a reduction argument from the setup without blackboard to the setup with blackboard. Here, we obtain a sufficient condition on $\delta$, such that the dynamics ``without blackboard'' setup can be reduced to the problem setting of ``with blackboard''. In the case of two agents, the proof for this reduction uses the following fact established in Section~\ref{sec:board}: in the absence of the black-board, Agent 2 requires at-most $k$ time-steps to infer the state of agent $1$. Thus, if $\delta^{'} = \delta/k$, then the deviation in arm-means in the time before communication can occur is at-most $\delta$. This coincides with the deviation of the setting ``with blackboard'' where in one time-step Agent 2 learns of the state of Agent 1.

Thus, the regret proofs for the case ``without blackboard'' are just corollaries of the regret proof ``with blackboard'' with $\delta'$. 

\subsection{Proof of Lemma~\ref{remark:remove_black_board}}
The proof of Lemma~\ref{remark:remove_black_board} follows identical argument. With the modified reward model, we argue in Section~\ref{sec:board} that it takes at most $k$ time steps for Agent $r$ to learn the arms that are being dominated by Agents ranked $1$ to $r-1$. Hence, essentially, the framework is equivalent to the proof of Lemma~\ref{rem:slow}, and hence the lemma follows. 
\section{Experimental Setup}
\label{sec:additional_experiments}

In Figure \ref{fig:simualtions}, we show through simulations  that - {\em (i)} Snooze-IT of \cite{krishnamurthy2021slowly} outperforms vanila UCB of \cite{auer2002finite} in the case of single agent, {\em (ii)} DNCB multi-agent setting is effective to simulate and matches the theoretical insights, and {\em (iii)} in the multi-agent case, DNCB outperforms UCB-D3, especially for higher ranked agents. 
In all settings, we consider the arms to have gaussian distribution with variance $0.4$ and means varying with time as given below. All plots are plotted after averaging over 10 runs, with the median being highlighted in bold, and the inter-quartile range between the $25^{th}$ and $75^{th}$ quantiles in the shaded region. In the single agent setting of Figure \ref{fig:single_agent}, we considered three arms, with the third arm having a fixed mean of $0.5$ throughout. In the multi-agent setting in Figures \ref{fig:multi_agent}, \ref{fig:multi_agent_2}, \ref{fig:dncb_d3}, \ref{fig:dncb_d3_2}, we initialized the arm means randomly from the uniform distribution on $[0,1]$. In each of the $10$ runs, the arm means for every agent-arm pair evolved independently according to a  symmetric random walk by either adding or subtracting a value of $\delta$ as specified in the plot title. We simulated the DNCB algorithm by assuming access to a black-board, the performance on which can be translated to the setting without access to the black-board as seen in Remark \ref{remark:remove_black_board}. For UCB-D3, we use the standard hyper-parameters recommended in \cite{ucbd3}. The plots in Figure \ref{fig:multi_agent} is averaged over the randomness in the arm-mean variation across time as well.


\end{document}